\documentclass{article}

\usepackage{booktabs}
\usepackage{makecell} 
 \usepackage[preprint]{neurips_2026}

\usepackage[utf8]{inputenc} % allow utf-8 input
\usepackage[T1]{fontenc}    % use 8-bit T1 fonts
\usepackage{hyperref}       % hyperlinks
\usepackage{url}            % simple URL typesetting
\usepackage{booktabs}       % professional-quality tables
\usepackage{amsfonts}       % blackboard math symbols
\usepackage{nicefrac}       % compact symbols for 1/2, etc.
\usepackage{microtype}      % microtypography
\usepackage{xcolor}         % colors
\usepackage{wrapfig}
\usepackage{subcaption}

\usepackage{amsmath,amsfonts,bm,mathtools}

\def\eqref#1{equation~\ref{#1}}
\def\1{\bm{1}}

\newcommand{\defeq}{\coloneqq}

\DeclareMathAlphabet{\mathsfit}{\encodingdefault}{\sfdefault}{m}{sl}
\SetMathAlphabet{\mathsfit}{bold}{\encodingdefault}{\sfdefault}{bx}{n}

\usepackage{amsmath}
\usepackage{amssymb}
\usepackage{mathtools}
\usepackage{amsthm}
\usepackage{mathrsfs}

\usepackage{thmtools}
\usepackage{thm-restate}

\usepackage[capitalize,noabbrev]{cleveref}

\theoremstyle{plain}
\newtheorem{theorem}{Theorem}[section]

\newtheorem{lemma}[theorem]{Lemma}
\newtheorem{corollary}[theorem]{Corollary}
\theoremstyle{definition}
\newtheorem{definition}[theorem]{Definition}
\newtheorem{assumption}[theorem]{Assumption}
\theoremstyle{remark}

\DeclareMathOperator{\rank}{rank}
\DeclareMathOperator{\cl}{cl}
\DeclareMathOperator{\LSE}{LSE}

\usepackage[textsize=tiny]{todonotes}

\title{A Unified Geometric Framework for Weighted Contrastive Learning}

\author{%
  Raphaël Vock \quad
  Edouard Duchesnay \quad
  Benoit Dufumier \\
  GAIA Lab, NeuroSpin, CEA, CNRS \\
  Université Paris-Saclay, Gif-sur-Yvette, France \\
  \texttt{\{raphael.vock,edouard.duchesnay,benoit.dufumier\}@cea.fr}
}

\begin{document}

\maketitle

\begin{abstract}
    Contrastive learning (CL) aims to preserve relational structure between samples by learning representations that reflect a similarity graph. Yet, the geometry of the resulting embeddings remains poorly understood. Here we show that weighted InfoNCE objectives can be interpreted as Distance Geometry Problems, where the weighting scheme specifies the target geometry to be realized by the representation. This viewpoint yields exact characterizations of the optimal embeddings for several supervised and weakly supervised objectives. In supervised classification, both SupCon and Soft SupCon (a dense relaxation of it where pairs from distinct classes have small non-zero similarity) collapse samples within each class to a single prototype. However, while balanced SupCon recovers the classical regular simplex geometry, class imbalance breaks this symmetry: SupCon induces non-uniform inter-class similarities depending on class sizes, whereas Soft SupCon preserves a regular simplex geometry regardless of class imbalance. In continuous-label settings, our framework reveals a different failure mode: $y$-Aware CL generally cannot attain its entropic optimum unless the labels lie on a hypersphere, exposing a mismatch between Euclidean label weights and spherical latent similarity. By contrast, geometrically consistent choices such as Euclidean–Euclidean weighting or $\mathbb{X}$-CLR admit unique optimal embeddings. Our results show that the choice of weighting scheme determines whether contrastive learning is geometrically realizable, degenerate, or inconsistent, providing a principled framework for designing contrastive objectives.
\end{abstract}

\section{Introduction}

\begin{figure}[h]
    \centering
    \includegraphics[width=.9\linewidth]{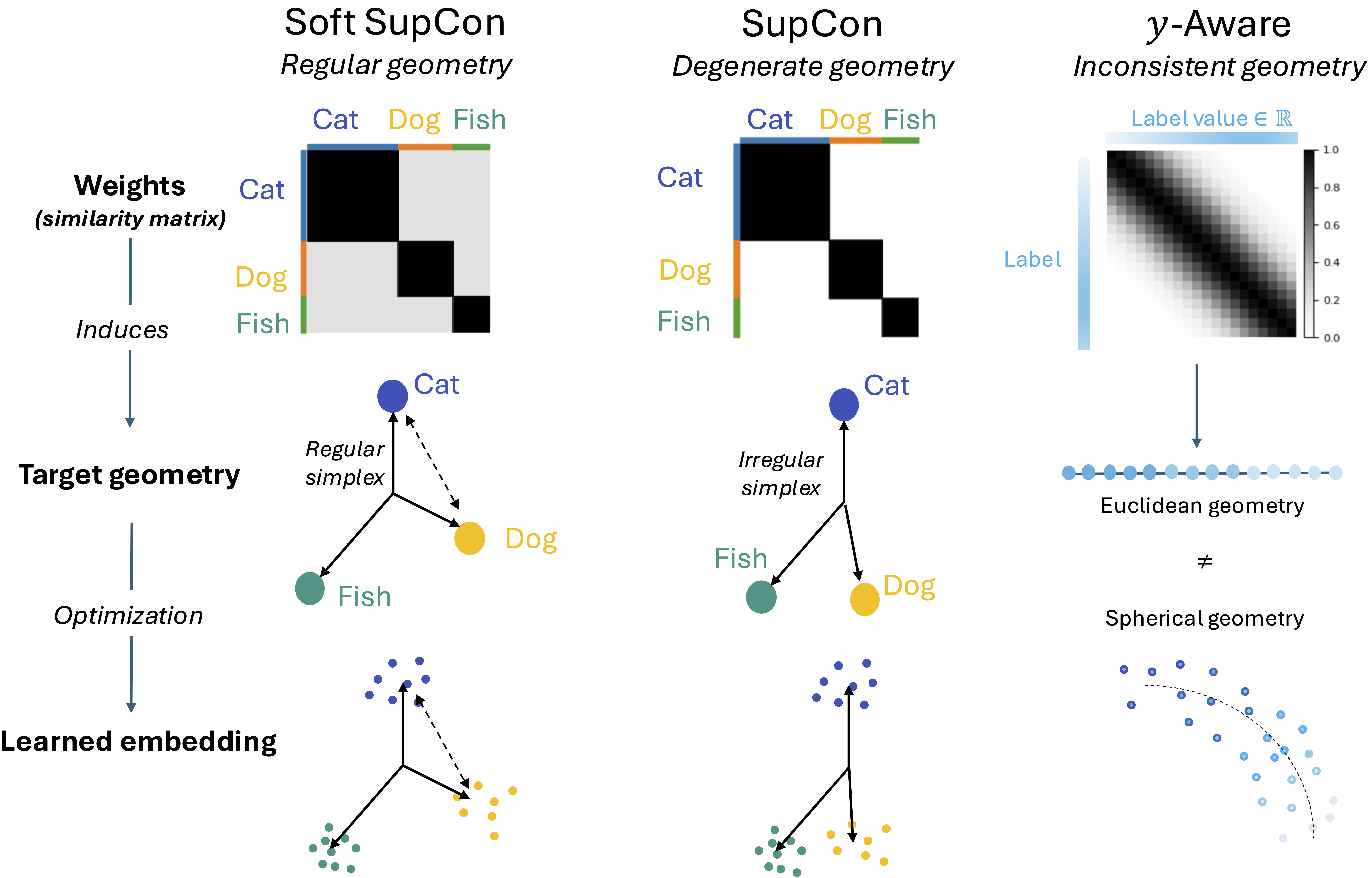}
    \caption{The weighting matrix (top) defines pairwise similarities that induce a target geometry (middle), which the embedding attempts to realize (bottom). Left: dense weights (Soft SupCon) yield a regular, symmetric geometry. Middle: sparse weights (SupCon) produce degenerate solutions, where class imbalance distorts prototype geometry. Right: continuous Euclidean label similarities ($y$-Aware) induce a geometry that is inconsistent with cosine similarities, leading to distortion.}
    \label{fig1}
\end{figure}

% Context: Empirical success CL but poor geometrical understanding
Contrastive Learning (CL) plays a pivotal role in current supervised and self-supervised deep representation learning models. It was popularized through several successful frameworks: SimCLR~\cite{chen2020simple} and MoCo for self-supervised learning of visual representations, SupCon~\cite{khosla2020supervised} for supervised classification, and CLIP~\cite{radford2021learning} for vision-language representation learning. Later on, it was adapted to perform supervised regression~\cite{zha2023rank, barbano2023contrastive}, handling continuous targets and weakly supervised learning~\cite{dufumier2021contrastive, schneider2023learnable} when auxiliary information (or \textit{meta-data}) is available during training as a means of constraining representations. Despite their empirical success, the geometry of the representations learned by contrastive objectives remains poorly understood. In particular, it is unclear what structure these objectives are implicitly trying to realize in the embedding space, and whether this structure can always be achieved.

% Context: CL from a geometric perspective
From a geometric perspective, CL can be viewed as attempting to realize a set of pairwise relationships between samples in a low-dimensional space. These relationships are specified implicitly by the weighting scheme used in the loss, which encodes prior knowledge about which samples should be similar or dissimilar. This raises a fundamental question: \emph{does there exist a representation satisfying these constraints, and if so, is it unique?}

%In essence, contrastive learning aims to preserve the relationship between input samples in an abstract representation space. From a graph-based perspective~\cite{cabannes2023active}, it learns to attract samples in the same neighborhood while repelling dissimilar samples in order to respect a similarity graph. This is a fundamental shift from previous reconstruction-based or cross-entropy-based objectives for which the mapping between the input and output is explicitly learned. In the language of \cite{shepard1975internal, shepard1981psychological}, CL learns a ``second-order isomorphism", i.e. the relationships between the external objects in the world, akin to our mental representations. 

% Current theoretical limitation: lack of precise characterization  
%In order to learn the similarity graph between samples, several weighted versions of the InfoNCE loss have been proposed depending on the prior knowledge available during training: class labels~\cite{khosla2020supervised}, latent class labels~\cite{arora2019theoretical}, textual descriptions~\cite{sobal2024}, and meta-data~\cite{dufumier2021contrastive}.

In this work, we show that weighted contrastive learning can be interpreted as a Distance Geometry Problem~\cite{liberti2014euclidean} (DGP), where the weighting scheme defines a target pairwise geometry. Under this view, learning representations amounts to finding a geometric realization of this target structure. This perspective reveals three regimes: some objectives define realizable geometries with unique solutions, others lead to degenerate solutions, and some are geometrically inconsistent.

%In this work, we take a step forward by giving a general formulation of the weighted InfoNCE objective and relating it to the Distance Geometry Problem (DGP)~\cite{liberti2014euclidean}, well-known in data visualization (MDS~\cite{chen2008multidimensional}) and wireless sensor network localization, among others areas. 

In supervised classification, we prove that both SupCon and a soft relaxation of it collapse samples within each class, but differ under class imbalance: SupCon induces class-size-dependent prototype configurations, whereas the soft variant preserves a regular simplex geometry independent of imbalance. In continuous-label settings, we show that $y$-Aware CL~\cite{dufumier2021contrastive} is generally inconsistent, while geometrically consistent formulations such as $\mathbb{X}$-CLR~\cite{sobal2024} admit realizable and unique optima.

Beyond theory, this perspective provides practical guidance for designing contrastive objectives by matching the geometry induced by the weights with that of the embedding space. We also introduce metrics to evaluate convergence to the predicted optima and validate our results experimentally.

In summary, we make the following contributions:

% \begin{itemize}
%   \item We show that state-of-the-art losses proposed for supervised and weakly-supervised contrastive learning (SupCon, $\mathbb{X}$-CLR, $y$-Aware CL) can be interpreted as solving a Distance Geometry Problem associated with a dissimilarity function induced by a weighting scheme.
%   \item We characterize the existence and uniqueness of the solution, i.e. the geometry of the optimal embedding, using the well-known theory of Euclidean Distance Matrices (EDM).
%   \item We provide a novel lower bound for the value of the loss and provide three new metrics, all geometric in nature, as a principled means to evaluate the convergence of models and the optimality of their representations. 
%   \item We provide empirical evidence for the theory at different scales and across datasets (MNIST, CIFAR, ImageNet).
% \end{itemize}

\begin{itemize}
    \item We show that weighted InfoNCE defines a target pairwise geometry and that minimizing the loss amounts to solving a DGP.
    \item We prove that both SupCon and Soft SupCon (a smooth relaxation in which inter-class pairs are assigned a small but non-zero similarity) collapse each class to a prototype, but differ under imbalance: SupCon produces class-size-dependent prototype geometry, whereas Soft SupCon preserves regular simplex geometry.
    \item We show that the label-space geometry and latent-space similarity must be matched for the optimum to be realizable (e.g. as in $\mathbb{X}$-CLR). As a consequence, we prove that $y$-Aware CL generally cannot attain its entropic lower bound because it combines Euclidean label distances with spherical latent similarity.
    \item We provide a novel lower bound for the value of the loss and provide three new metrics, all geometric in nature, as a principled means to evaluate the convergence of models and the optimality of their representations. 
\end{itemize}

\section{Related Work}

Theoretical aspects of CL have mainly been studied from three perspectives: information theory, probability and graph theory, and metric learning. Our approach is in essence a geometric one and mostly aligns with this last view of CL. 

% First MI view of CL (Oord, Hjelm, Bachman) with some bounds (Poole) on the loss : not enough to characterize the geometry or explain dowsntream performance (Tschanenn). 
\textbf{Mutual information perspective.} Early modern contrastive objectives were motivated by the fact that they can be understood as tractable proxies for mutual information (MI) maximization. Contrastive Predictive Coding~\cite{oord2018cpc} introduced the InfoNCE loss and showed how negative sampling yields a scalable lower bound on MI in sequential prediction settings. In parallel, Deep InfoMax proposed MI maximization between global and local feature statistics for unsupervised representation learning~\cite{hjelm2018deepinfomax}, and related multi-view MI objectives were developed for self-supervised learning~\cite{bachman2019amdim}. Subsequent work clarified the behavior and limitations of variational MI bounds (including InfoNCE) through bias–variance tradeoffs~\cite{poole2019variational}. However, MI does not fully explain representation quality or geometry: empirical evidence~\cite{tschannen2020mutualinfo} shows that downstream performance strongly depends on architectural and estimator-induced inductive biases, rather than on MI estimation alone.

% Probabilistic view of CL: most of them in self-supervised case, Saunshi treated it with the concept of latent classes then HaoChen refined it with the notion of augmentation graph and its partition, a weaker formulation. Saunshi later argued it was impossible to characterize the optimum representation explaining downstream performance if no constraints is set on the family of encoders. Wang & Isola treated the ssl case by splitting the loss into two parts, one aligment term and uniformity term. however, as argued by the authors, it is not possible to find an encoder that is both perfectly aligned and uniform with a finite number of data, limiting their analysis. 
\textbf{Probabilistic and graph-based view.} Another line of research models positives as samples sharing an unobserved semantic factor (latent class), yielding generalization guaranties for downstream classification from contrastive pretraining~\cite{arora2019theoretical}. More recent analyses shift from strong assumptions on latent-class sampling to augmentation-driven positive pairs by introducing an augmentation graph whose spectral structure is related to contrastive objectives~\cite{haochen2021provable, haochen2022theoretical}. Stronger results were obtained on downstream performances with this formulation by assuming semantic clustering of this augmentation graph. Meanwhile, an attempt to unify contrastive and non-contrastive approaches was proposed~\cite{balestriero2022contrastive} based on a ``similarity graph'', a rough approximation of the augmentation graph. A complementary view is given in ~\cite{wang2020alignment} which interprets contrastive learning as trading off an alignment term (pulling positives together) and a uniformity term (uniformly spreading representations out on a hypersphere), while also noting the inherent limitations to learning this trade-off with limited data.

%At the same time, it has been argued that understanding contrastive learning—and in particular relating loss minimization to downstream performance—requires explicitly accounting for inductive biases such as the model class and optimization, since function-class-agnostic analyses can become vacuous in realistic regimes~\cite{saunshi2022inductive}. 

% More geometrical perspective of supervised contrastive learning is given by Graff, where they show that optimal hyperspherical representations of SupCon loss for a regular simplex. This is the first caracterization of the geometry of a contrastive loss in the supervised case with discrete labels. 
\textbf{Metric Learning.} Before the MI and probabilistic/graph-based perspectives of CL, early works on dimensionality reduction~\cite{hadsell2006dimensionality} proposed a more geometric formulation of CL with subsequent pairwise and triplet-based losses, which explicitly enforce geometric structure in embedding space. Modern self-supervised~\cite{chen2020simple, he2020momentum} CL frameworks can be seen as scalable extensions of this metric learning view that implicitly shape the geometry of representations on the unit sphere. In the discrete supervised setting, SupCon~\cite{khosla2020supervised} extends self-supervised objectives using class labels to define positive samples. A more explicit geometric characterization was provided by~\cite{graf2021dissecting}, which analyzes the loss-minimizing configurations of SupCon and states that, under idealized conditions, class representations approach a regular simplex arrangement on the hypersphere. In the regression case, kernel-based variants~\cite{dufumier2021contrastive, barbano2023contrastive} and rank-based variants~\cite{zha2023rank} of the InfoNCE objective have been derived, but as far as we are aware a geometrical characterization of their optima has not been given until now.

% Overall, while contrastive learning has been studied actively in self-supervised and supervised cases, the geometrical realization of its optimal representation is still not well understood, especially compared to more classical frameworks with cross-entopy loss (discrete case) or l1/l2 loss (continuous case). 
Overall, while contrastive learning has been studied extensively from MI and probabilistic/graph-based viewpoints, a precise characterization of the geometry of optimal representations (including existence, uniqueness and the resulting distance structure) remains as of yet unexplored.

\section{Analysis of Weighted Contrastive Learning}

\subsection{Unified formulation of weighted Contrastive Learning losses}

\paragraph{Problem setup.} The goal of representation learning~\cite{bengio2013representation} is to learn a mapping $f: \mathcal{X} \rightarrow \mathcal{Z}$ from input data $X=(x_i)_{i\in [1..n]} \in \mathcal{X}^n \subseteq \mathbb{R}^{n\times p}$ such that its representation $Z=(z_i)_{i\in [1..n]}:= f(X) \in \mathcal{Z}^n \subseteq \mathbb{R}^{n\times q}$ retains useful information for a variety of downstream tasks. 
% In all that follows, $\|\cdot\|$ denotes the Euclidean norm, $\cos$ is the cosine similarity between a pair nonzero vectors, $I_n$ the $n\times n$ identity matrix, and $E_n$ the $n\times n$ matrix with unit coefficients. 

%To show the relationship between InfoNCE and MDS, we will write the InfoNCE loss as a simple cross-entropy loss between two distributions. Then, we will demonstrate that these distributions match when an underlying MDS problem is solved with a specific distance metric (or equivalently similarity metric). 

\begin{definition}[$w$-InfoNCE loss]
    Given a similarity matrix $W=(w_{ij}) \in \mathbb{R}^{n\times n}$ such that $w_{ij} \ge 0$ for $i\neq j$, and $S=(s_{ij})$ where $s_{ij} = s(z_i, z_j)=s(f(x_i), f(x_j))$ for $s: \mathbb R^q\times \mathbb R^q \to \mathbb R$ some measure of similarity in the latent space, the \textit{weighted InfoNCE loss} is defined as
    \begin{equation*}
        \mathcal{L}_{\mathrm{NCE}}^W := -\frac{1}{n}\sum_{i=1}^n \sum_{j \neq i} \frac{w_{ij}}{\sum_{k\neq i} w_{ik}} \log\left(\frac{\exp{s_{ij}}}{\sum_{k \neq i} \exp{s_{ik}}} \right)
    \end{equation*}
\end{definition}

\textbf{Remark.} Originally, SupCon, $y$-Aware and $\mathbb{X}$-CLR all define the InfoNCE loss in terms of several augmentations of every given sample. We argue this case can be recovered from the above formulation by concatenating several data augmentations in a single matrix $X$.

Let us we recall the weighting schemes that correspond to several well-known frameworks:
\begin{itemize}
    %\item SimCLR~\cite{chen2020simple}: $w_{ij}=1$ if $(x_i, x_j)$ are transformations of the same sample and 0 otherwise. 
    \item SupCon~\cite{khosla2020supervised}: $w_{ij}=1$ if $y_i=y_j$ and 0 otherwise for a labeled dataset $(X, Y)$ with $Y\in [1..C]^n$ and $C$ the number of classes.
    \item \textit{y}-Aware~\cite{dufumier2021contrastive}: $w_{ij}=\exp{\left(-\|y_i-y_j\|^2\right)}$ for a weakly-labeled dataset $(X, Y)$ with $Y\in \mathbb{R}^{n\times \ell}$ the meta-data. 
    \item $\mathbb{X}$-CLR~\cite{sobal2024}: $w_{ij}=\exp(\cos(y_i, y_j)/\tau')$ for an image-caption dataset $(X,Y)$ with $Y$ the textual descriptions of the images encoded by a pretrained text encoder. 
\end{itemize}
In each of the above cases, similarity in the latent space is measured with cosine similarity: $s_{ij}=\cos(z_i, z_j)/\tau$, with $\tau$ a temperature hyperparameter.

% We make one important assumption on $S$ and $W$ before stating our main theorems:

\begin{assumption}[Symmetry of $S$ and $W$]
    The similarity function $s$ and similarity matrix $W$ are symmetric, so $s_{ij}=s_{ji}$ and $w_{ij}=w_{ji}$ for all $i, j\in[1..n]$. 
    \label{hyp:symmetry}
\end{assumption}

We divide the analysis into two parts. First, we study the \emph{soft} weighting regime, where $W$ is strictly positive (§~\ref{sec:entropy_lower_bound}--\ref{sec:sup_regression}). We then turn to the \textit{hard} or \textit{sparse} regime, where $W$ may contain zeros (§~\ref{sec:sparse_weights}--\ref{sec:supcon}).

\subsection{Entropic lower bound}
\label{sec:entropy_lower_bound}
% \begin{assumption}[Strict positivity of $W$]\label{assumption:w_positive}
%     Off-diagonal terms in $W$ have strictly positive values: $w_{ij}>0$ for $i\neq j$.
%     \label{hyp:w_positive}
% \end{assumption}

Under a strict-positiveness assumption, we can state the entropic lower bound of the $w$-InfoNCE loss (proof in Appendix~\ref{proof:optimal_embedding}):
\begin{restatable}[Entropic lower bound of $w$-InfoNCE]{theorem}{StatementInfonceOptimum}\label{th:infonce_optimum}
    Under \cref{hyp:symmetry} and assuming $w_{ij}>0$ for $i\neq j$, viewing the weighted InfoNCE loss $\mathcal L_\mathrm{NCE}^W$ as a function of $S=(s_{ij})$ over the space of symmetric similarity matrices, the global minimum is attained precisely when $$s_{ij}^* = \log(w_{ij})+c \qquad \text{for}\ i \neq j$$ for some constant $c \in \mathbb R$. Moreover, that minimum equals
    \begin{equation}\label{eq:info_nce_min}
        \min_S\mathcal L_\mathrm{NCE}^W = -\frac{1}{n}\sum_{i \neq j} \frac{w_{ij}}{\sum_{k\neq i} w_{ik}} \log\left(\frac{w_{ij}}{\sum_{k \neq i} w_{ik}} \right).
    \end{equation}
\end{restatable}

This result echoes what has been previously obtained using a probabilistic framework~\cite{poole2019variational}, with $s$ being the ``critic'' discriminating positive pairs against negative ones. In our case, we obtain a geometric characterization of $Z^*$ which is more precise than previous works\footnote{In~\cite{poole2019variational}, the constant $c$ depends on $i$.} (in part thanks to Assumption~\ref{hyp:symmetry}) and which does not depend on a dichotomy between positive and negative samples as in the classical view of contrastive learning. %Note that previous work implicitly took Assumption~\ref{hyp:w_positive} for granted without discussing its validity in real-case scenarios.

\subsection{Connection with the Distance Geometry Problem}

We can now express the relationship between the $w$-InfoNCE loss and the Distance Geometry Problem (DGP)~\cite{liberti2014euclidean} as immediate corollaries of \cref{th:infonce_optimum}:

\begin{restatable}{corollary}{StatementEuclideanCorollary}\label{corr:optimal_embedding}
    When latent similarity is given by $s_{ij} = - \|z_i - z_j\|^2$ and $w_{ij}=\exp(-d_{ij})$ for a symmetric dissimilarity matrix $D=(d_{ij})$, attaining the entropic lower bound of $\mathcal L_{\mathrm{NCE}}^W$ stated in \cref{th:infonce_optimum} amounts to solving the DGP
    \begin{equation}
        \|z_i^* - z_j^*\|^2 = d_{ij} + c \qquad \textrm{for } i\neq j,
        \label{eq:mds_pb}
    \end{equation}
    for any constant $c\in \mathbb{R}$.
    \label{corr:infonce_eq_mds}
\end{restatable}

\begin{restatable}{corollary}{StatementSphericalCorollary}
    When $s_{ij} = \cos(z_i,z_j)/\tau$ and $w_{ij}=\exp(-{d_{ij}})$ for a symmetric dissimilarity matrix $D=(d_{ij})$, attaining the entropic lower bound of $\mathcal L_{\mathrm{NCE}}^W$ amounts to solving the cosine problem
    \begin{equation}
        \cos(z_i, z_j) = -\tau d_{ij} + c'\qquad \textrm{for } i \neq j,
        \label{eq:spherical_mds_pb}
    \end{equation}
    for any constant $c' \in \mathbb R$.
\label{corr:infonce_eq_spherical_mds}
\end{restatable}

These corollaries show that minimizing $w$-InfoNCE amounts to finding a low-dimensional embedding with prescribed similarity, up to an additive constant $c$ or $c'$. We will go on to show that a mild condition on $n$ and $\rank(D)$ suffice to entail that $c=0$ or $c'=1$.

\subsection{Existence and uniqueness of the optimal embedding}

Having established the link between DGP and weighted InfoNCE, we ask ourselves the question of existence and uniqueness of solutions to these two families of problems. The answers involve a classical notion from computational geometry~\cite{alfakih2018euclidean}. 

\begin{definition}(EDM)
    An $n\times n$ matrix $D=(d_{ij})$ is called an \textit{Euclidean Distance Matrix} (EDM) if there exist points $z_1^*,\dots,z_n^*\in \mathbb{R}^q$ such that
    \begin{equation}
        \|z_i^*-z_j^*\|^2=d_{ij} \qquad \text{for}\ i,j \in [1..n].
    \end{equation}
    We call the points $z_1^*, \ldots, z^*_n$ a \emph{$q$-dimensional realization} of $D$. The \emph{embedding dimension} is defined as the smallest dimension $r$ of a Euclidean realization, or equivalently as the dimension of the affine span of any given Euclidean realization. Distinct Euclidean realizations of an EDM in a given dimension are congruent, i.e. they differ by a Euclidean isometry. A minimal Euclidean realization (i.e. in $\mathbb{R}^r$) is called a set of \emph{generating points} of $D$. 
    Finally, we say that $D$ is \textit{spherical} if those generating points can be chosen such that $\|z_1\|=\cdots=\|z_n\|=\rho$ for some radius $\rho > 0$. The smallest such $\rho$ is called the \textit{radius} of a spherical EDM.
\end{definition}

% The following result, proved in Appendix~\ref{proof:mds_solution}, states that given enough data, the dissimilarities $(d_{ij})$ must form an EDM and the constant $c$ must vanish in order the $w$-InfoNCE loss to attain its entropic lower bound.

\begin{restatable}[Euclidean DGP]{theorem}{StatementDGPSolution}
    \label{th:mds_solution}
    Assume that the number of points $n$ is strictly greater than $2q+3$. Assume that $s_{ij} = \exp(-\|z_i-z_j\|^2)$ and $W=(\exp(-d_{ij}))$ for $D=(d_{ij})$ an $n\times n$ dissimilarity matrix with $\rank(D)\le q+2$. Then the entropic lower bound of $\mathcal L_{\mathrm{NCE}}^W$ is attained iff $D$ is EDM of embedding dimension $r\le q$, up to diagonal elements. In that case, the minimum $Z^*$ satisfies $\|z^*_i-z^*_j\|=d_{ij}$ for all $i,j \in [1..n]$ and is essentially unique (i.e. up to Euclidean isometry).
\end{restatable}

Next, we consider the case of Corollary~\ref{corr:infonce_eq_spherical_mds} where the similarity metric between latent vectors is given by cosine similarity. This is the most common use-case in CL~\cite{khosla2020supervised, dufumier2021contrastive, sobal2024} and it has a rigidity property analogous  to \cref{th:mds_solution}, which we prove in Appendix~\ref{proof:spherical_mds_solution}: 
\begin{restatable}[Spherical DGP]{theorem}{StatementSphericalDGPSolution}
    \label{th:spherical_dgp_solution}
    Assume that $n > 2q+4$. Assume that $s_{ij}=\exp(\cos(z_i,z_j)/\tau)$ and $w_{ij}=\exp(-d_{ij})$ for $D=(d_{ij})$ an $n\times n$ dissimilarity matrix with $\rank(D)\le q+1$. Then the entropic lower bound of $\mathcal L_{\mathrm{NCE}}^W$ is attained by some embedding $Z^*\in \mathbb{R}^{n \times q}$ iff $D$ is a spherical EDM (up to diagonal elements) with embedding dimension $r\le q$ and radius $\rho\le 1/\sqrt{2\tau}$ (with equality when $r=q$). In that case, $\cos(z_i,z_j)=1-\tau d_{ij}$ for all $i,j \in [1..n]$ and $Z^*$ is essentially unique (i.e. the normalized vectors $\widetilde z_{i}:=z_i^*/\|z_i^*\|$ are unique up to linear Euclidean isometry).
\end{restatable}

\subsection{Supervised classification: Soft SupCon}\label{ss:soft-supcon}

We now show that the Soft SupCon weighting scheme arises as a direct application of \cref{th:spherical_dgp_solution}.

\begin{corollary}[Soft SupCon]\label{cor:soft_supcon}
Let $y_i \in [1..C]$ be one of $C$ discrete classes associated with each sample and define $w_{ij}=1$ if $y_i=y_j$ and $w_{ij}=\varepsilon \in (0,1)$ otherwise. 
Assume $n>2q+4$ and consider $s_{ij}=\exp(\cos(z_i,z_j)/\tau)$. 
If $C \le q$ and $\tau \le (C-1)/(-C\log\varepsilon)$, then the entropic lower bound of $\mathcal L_{\mathrm{NCE}}^W$ is attained by some $Z^*\in\mathbb{R}^{n\times q}$, which is unique up to linear Euclidean isometry of the normalized embeddings. At the optimum, one has $\cos(z_i,z_j)=1$ when $y_i=y_j$, and $\cos(z_i,z_j)=1+\tau\log\varepsilon$ otherwise. 
In particular, all points from the same class collapse to a single representation, and the $C$ classes form a simplex with uniform inter-class similarity $\beta=1+\tau\log\varepsilon$. 
The simplex is regular when $\tau^*=(C-1)/(-C\log\varepsilon)$, corresponding to $\beta^*=-1/(C-1)$.
\end{corollary}

% This result shows that Soft SupCon induces a regular simplex geometry that is independent of class sizes, in contrast with the sparse SupCon regime analyzed later.

\subsection{Supervised multivariate regression}
\label{sec:sup_regression}
Next, we state three results corresponding to various weighting schemes that can be encountered when using continuously labeled data (see Appendix~\ref{proof:sup_continuous_case} for the proofs). Let $(X, Y)$ be a labeled dataset with $\ell$ continuous labels $Y\in \mathbb{R}^{n\times \ell}$. We assume throughout that $n>2q+4$ so that \cref{th:mds_solution} and \ref{th:spherical_dgp_solution} can be applied. Our first result concerns the case where Euclidean distance is used as a dissimilarity function in both label space and representation space.

\begin{restatable}[Euclidean $w$-InfoNCE]{theorem}{StatementContinuousLabels}\label{th:continuous_labels}
    Assume $\ell \leq q$. If $w_{ij}=\exp({-\|y_i-y_j\|^2})$ and $s_{ij}=-\|z_i-z_j\|^2$, the $w$-InfoNCE loss attains its entropic lower bound in $Z^*=(\widetilde y_i)_{i=1}^n$, where $\widetilde y_i=(y_i, 0,\dots,0)\in \mathbb{R}^q$ for $i\in [1..n]$, which is essentially unique.
\end{restatable}
Our next result is a negative one and concerns the Euclidean--spherical $w$-InfoNCE loss, i.e. the $y$-Aware loss~\cite{dufumier2021contrastive}.

\begin{restatable}[$y$-Aware]{theorem}{StatementYAware}
\label{th:y_aware}
  Even if $\ell \leq q$, the $y$-Aware InfoNCE loss, which is the $w$-InfoNCE loss with $s_{ij}=\cos(z_i,z_j)/\tau$ and $w_{ij}=\exp(-\|y_i-y_j\|^2)$, does not attain the entropic lower bound unless the $y_i$ happen to lie on a hypersphere in $\mathbb{R}^\ell$.
\end{restatable}
% The intuition behind this negative result is that points on the hypersphere $\mathbb{S}^{q-1}$ need not be isometric to a set of points $\mathbb{R}^\ell$.

On the other hand, the spherical--spherical case, as in the $\mathbb{X}$-CLR loss~\cite{sobal2024}, does attain the entropic lower bound and offers a spherical counterpart to \cref{th:continuous_labels}.

\begin{restatable}[$\mathbb{X}$-CLR]{theorem}{StatementXCLR}
\label{th:x_clr}
    If $\ell < q$ and $\tau \leq \tau'$, the $\mathbb{X}$-CLR loss~\cite{sobal2024}, which is the $w$-InfoNCE loss with $s_{ij}=\cos(z_i,z_j)/\tau$ and $w_{ij}=\exp(\cos(y_i,y_j)/\tau')$ attains the entropic lower bound in the points $z_i^*=(y_i/\|y_i\|, \sqrt{\tau'/\tau-1}, 0, \dots,0)$ for $i\in [1..n]$. Moreover, that minimum is essentially unique.
\end{restatable}

In practice~\cite{sobal2024}, the dimension of text embeddings in $\mathbb{X}$-CLR is $\ell=768$ for CLIP and the latent dimension $q=2048$ for ResNet50. It was also found that $\tau'\approx\tau$ gives the best classification accuracy on ImageNet.  

\textbf{Practical implications.} \cref{th:y_aware} and~\ref{th:x_clr} suggest that the pair of similarity functions used to define $(W, S)$ should be chosen consistently to allow for a realizable optimum. The spherical--spherical or Euclidean-Euclidean weighting schemes are therefore preferable over the Euclidean--spherical scheme (as in $y$-Aware).

\subsection{Sparse weights and limiting properties}
\label{sec:sparse_weights}

The assumption of positive weights, i.e. $w_{ij} > 0$ in~\cref{th:infonce_optimum}, may be relaxed to $w_{ij} \geq 0$ assuming that $\sum_{ik}w_{ik}>0$ so that $\mathcal L_{\mathrm{NCE}}^W$ is well defined: we term this the \textit{sparse case}. Although any sparse matrix $W$ can be understood as a limit of a sequence of positive matrices $(W_m)$, the optimum of $\mathcal L_{\mathrm{NCE}}^W$ need not be the limit of the optimum of $\mathcal L_{\mathrm{NCE}}^{W_m}$ because in general convergence is not uniform. Moreover, the configuration of points which attain the minimum may behave qualitatively different in the limit.

In Appendix~\ref{appdx:sparse_suboptimality} we prove that in the sparse case, the entropic lower bound continues to hold and that it is sharp. We give a necessary and sufficient for it to be optimal (i.e. impossible to improve upon). We go on to show that it is optimal in the case of the so-called Euclidean SupCon (i.e. the $w$-InfoNCE loss with a hard SupCon weighting scheme and Euclidean similarities $s_{ij}=-\|z_{i}-z_j\|^2$) but that quasi-optima can be produced in any number of qualitatively distinct configurations. This comes in stark contrast to the soft case where optima are essentially unique (\cref{th:mds_solution}).

\subsection{SupCon}
\label{sec:supcon}

We now consider the case of SupCon proper (i.e. the $w$-InfoNCE loss with a hard SupCon weighting scheme and spherical similarities $s_{ij}=\cos(z_i,z_j)/\tau$). Using a compactness argument, we see in Appendix~\ref{appdx:supcon} that the entropic lower bound is sub-optimal. We go on to prove the following theorem:
\begin{restatable}{theorem}{StatementSupconMinimum}\label{th:supcon-minimum}
Let $C$ denote the number of discrete classes and $\ell_c$ the size of the $c$th class. Assume that $C \leq q$.
\begin{enumerate}
    \item The SupCon loss has a unique global minimum over the matrices of the form $G/\tau$, with $G$ an $n \times n$ cosine matrix.
    \item That minimum has collapsed representations and can be expressed in terms of class prototypes $\mu_1, \ldots, \mu_C \in \mathbb S^{q-1}$ which are unique up to linear Euclidean isometry.
    \item  If $c \neq c'$, the value of $\cos(\mu_c,\mu_{c'})$ depends only on $(\ell_c, \ell_{c'})$.
\end{enumerate}
\end{restatable}
In particular we see that SupCon suffers from representation collapse which generalizes both \cite{nguyen2024neural} and \cite{behnia2024supervised} in that we make no assumptions on $\tau>0$ nor do we require an infinite number of samples.

In the case of balanced classes, we easily recover the result from \cite{graf2021dissecting} that SupCon optima correspond to regular simplices (Appendix~\ref{appdx:supcon}, \cref{th:graf2021}). Moreover, Hard and Soft SupCon have similar optima: both collapse classes and yield uniform inter-class similarities. For Hard SupCon, $\beta_{\mathrm{hard}}=-1/(C-1)$ depends on $C$; for Soft SupCon, $\beta_{\mathrm{soft}}=1+\tau\log\varepsilon$ depends on $(\tau,\varepsilon)$ but not on $C$. Choosing $\tau^*=(C-1)/(-C\log\varepsilon)$ yields a regular simplex.

%It seems plausible that the technique of parameter reduction we used in our proof may be applied to the imbalanced case, but for now we leave this as an avenue to be explored in future work.

%Looking back at \S~\ref{ss:soft-supcon} it becomes apparent that in the balanced case, hard and soft SupCon have similar optima. Indeed, both suffer from class collapse and their optimal representations have uniform inter-class similarities. In the hard case, inter-class similarity is $\beta_{\mathrm{hard}}=-1/(C-1)$ which depends on $C$ but not on $\tau$. In the soft case, inter-class similarity is $\beta_{\mathrm{soft}}=1+\tau \log(\varepsilon)$ which depends both on $\tau$ and the soft negative weight $\varepsilon \in (0,1)$, but not on $C$. Choosing the maximal feasible temperature of $\tau^*=(C-1)/(-C\log\varepsilon)$ recovers hard SupCon's regular simplex geometry.

%This resemblance breaks down under imbalance classes. Indeed, class imbalance breaks SupCon's symmetry and causes its optima to have non-uniform inter-class similarities, tending to group minority classes closer together. Interestingly this is not the case for soft SupCon whose optima are utterly indifferent to class imbalance.

This equivalence breaks under class imbalance: SupCon produces non-uniform inter-class similarities, grouping minority classes more closely, whereas Soft SupCon is agnostic to class imbalance.

\subsection{Metrics}\label{ss:metrics}
    To empirically evaluate whether the similarity structure of the learned encodings $Z=(z_1;\ldots; z_n) \in \mathbb R^{n \times q}$ are close to the theoretical optimum $Z^* = (z_1^*; \ldots;z_n^*) \in \mathbb{R}^{n \times q}$ as predicted by any of our theoretical results, we introduce three metrics.

% The first is defined in terms of the difference between the loss and its lower bound as expressed in Theorem~\ref{th:infonce_optimum}:

\begin{definition}
    The \textit{$w$-InfoNCE entropic loss gap} is defined as the nonnegative quantity $\Delta_W:=({\mathcal{L}_{\mathrm{NCE}}^W}/{\min_S \mathcal L_\mathrm{NCE}^W})-1$,
    where $\min_S \mathcal L_\mathrm{NCE}^W$ is the entropic lower bound stated in \cref{th:infonce_optimum}.
%$$\Delta_W:=\frac{\mathcal{L}_{\mathrm{NCE}}^W}{\min_S \mathcal L_\mathrm{NCE}^W}-1,$$
\end{definition}
Note that a small value of $\Delta_W$ is a necessary but not sufficient condition for the $z_i$ to be close to the optimum $z_i^*$. This motivates the introduction of a second metric, related to Kruskal's stress, which directly compares the pairwise similarities of the $z_i$ versus those of the $z_i^*$.
\begin{definition}
    Let $s:\mathbb R^q \times \mathbb R^q \to \mathbb R$ denote a similarity function in latent space, e.g. $s(x,y) = -\|x-y\|^2$ or $s(x,y) = \cos(x,y)$. The \textit{coefficient of similarity} between $Z$ and $Z^*$ is defined as the $r^2$ of $s(z_i, z_j)$ as a predictor of $s(z_i^*, z_j^*)$:
    \begin{align*}
        r^2_{s\textrm{-sim}}(Z, Z^*) &:= 1-\frac{\widehat \mu_{ij} \left[(s(z_i, z_j) - s(z_i^*, z_j^*))^2\right]}{\widehat \sigma^2_{ij}[s(z_i^*, z_j^*)]}
    \end{align*}
where $\widehat \mu_{ij}$ and $\widehat \sigma^2_{ij}$ denote, respectively, sample mean and sample variance over the set of pairs $(i,j) \in [1..n]^2$.
\end{definition}
Note by standard results that if $s$ is the negative squared Euclidean distance then $r^2_{s\textrm{-sim}}=1$ iff $Z$ and $Z'$ differ by a Euclidean isometry. The same result holds if $s$ is cosine similarity and $Z$ and $Z^*$ are taken on the unit hypersphere. Note that the computation of $r^2_{s\textrm{-sim}}$ is $O(n^2)$ in time, and therefore impractical for large $n$. This motivates the following definition, which measures similarity at an intrinsic level in sub-quadratic $O(n\log n)$ time.

% \begin{definition}
%     Let $(\Omega, b) \in O_q(\mathbb R) \times \mathbb R^n$ be the solution to the \textit{affine orthogonal Procrustes problem}:
%     $$\begin{cases}
%     \textrm{for} &(\Omega, b)\in \mathbb R^{q\times q} \times \mathbb R^n,\\
%     \textrm{minimize} &\sum_{i=1}^n \|(\Omega z_i +b) - z_i^*\|^2\\
%     \textrm{subject to} & \Omega^T\Omega= I_q.
%     \end{cases}$$
% \end{definition}
% We then define the \textit{Procrustes similarity} as the $r^2$ of the best fit:
% $$r^2_{\mathrm{Proc}}(Z, Z^*) := 1-\frac{\widehat \mu_i \left[\|(\Omega z_i + b) - z_i^*\|^2\right]}{\widehat \sigma_i\left[z_i^*\right]}.$$

\begin{definition}
Let $(\Omega,b)$ minimize $\sum_{i=1}^n \|\Omega z_i + b - z_i^*\|^2$ over $\Omega \in O_q(\mathbb{R})$, $b \in \mathbb{R}^q$. The \textit{Procrustes similarity} is defined by
\[
r^2_{\mathrm{Proc}}(Z,Z^*)
= 1 - \frac{\widehat{\mu}_i \left[\|\Omega z_i + b - z_i^*\|^2\right]}
{\widehat{\sigma}_i^2[z_i^*]},
\]
\end{definition}
i.e. the coefficient of determination after optimal rigid alignment. Note that $r^2_{\mathrm{Proc}}(Z, Z^*) = 1$ iff $Z$ and $Z^*$ differ by a Euclidean isometry.

% \begin{restatable}{theorem}{StatementMetricTheorem}
%     Assume that $n > 2q+3$ and $s$ is either
% \end{restatable}

\section{Experiments}
Armed with the metrics defined in \S~\ref{ss:metrics}, we conduct a series of experiments to validate our theoretical claims and verify whether representations learned with $w$-InfoNCE loss on practical datasets are in fact near their geometric optima. Implementation details are provided in Appendix~\ref{appdx:exp_details}.

\subsection{Weighted InfoNCE on MNIST: discrete, continuous and mixed regime}
\label{sec:w_infonce_mnist}

\begin{wrapfigure}[13]{r}{0.5\textwidth}
    \vspace{-10pt}
    \centering
    \includegraphics[width=\linewidth]{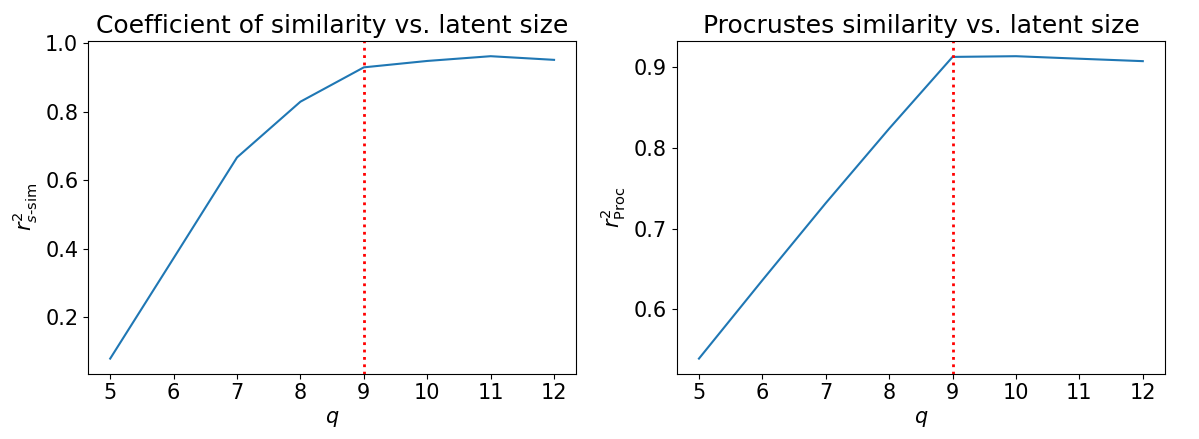}
    \caption{Euclidean representations of MNIST learned with the discretely weighted InfoNCE loss. Left: coefficient of similarity between $Z$ and regular 9-simplex versus latent dimension $q$. Right: Procrustes similarity versus $q$.}
    \vspace{-10pt}
    \label{fig:mnist_cat}
\end{wrapfigure}

We train a ConvNet with the $w$-InfoNCE loss on MNIST~\cite{mnist} across three regimes: 
(i) discrete classification,
(ii) continuous regression, and
(iii) mixed discrete--continuous labels.

\textbf{Discrete case.} We first consider the discrete classification setting, using Soft SupCon weights (as in \S~\ref{ss:soft-supcon}) with the value $\varepsilon=e^{-1}$ between negative samples. 
The optimal geometry corresponds to a regular 9-simplex, realizable in $\mathbb{R}^q$ for $q \geq 9$. 
As shown in Fig.~\ref{fig:mnist_cat}, both $r^2_{s\text{-sim}}$ and $r^2_{\mathrm{Proc}}$ increase with $q$ up to a sharp transition at $q=9$, where the learned representation becomes nearly isometric ($r^2_{\mathrm{Proc}} \approx 0.9$).

\begin{figure}[h]
    \centering
    %\hfill
    \begin{subfigure}{0.49\columnwidth}
        \centering
        \includegraphics[width=\linewidth]{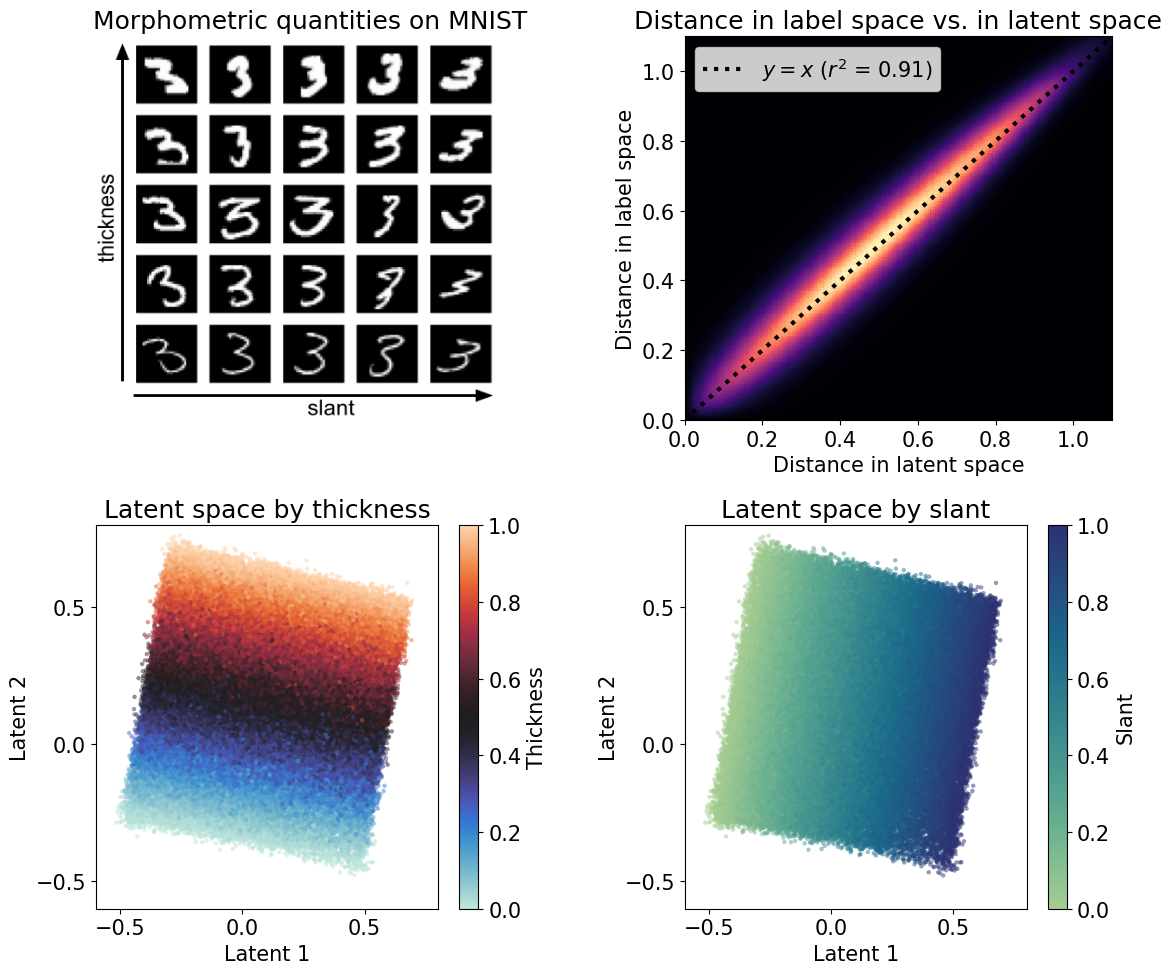}
        \caption{2D Euclidean representations of MNIST learned using $w$-InfoNCE with continuous 2D labels. Top: morphometric weights and distance comparison. Bottom: learned embeddings colored by thickness (left) and slant (right).}
        \label{fig:mnist_2d}
    \end{subfigure}
    \hfill
    \begin{subfigure}{0.49\columnwidth}
        \includegraphics[width=\linewidth]{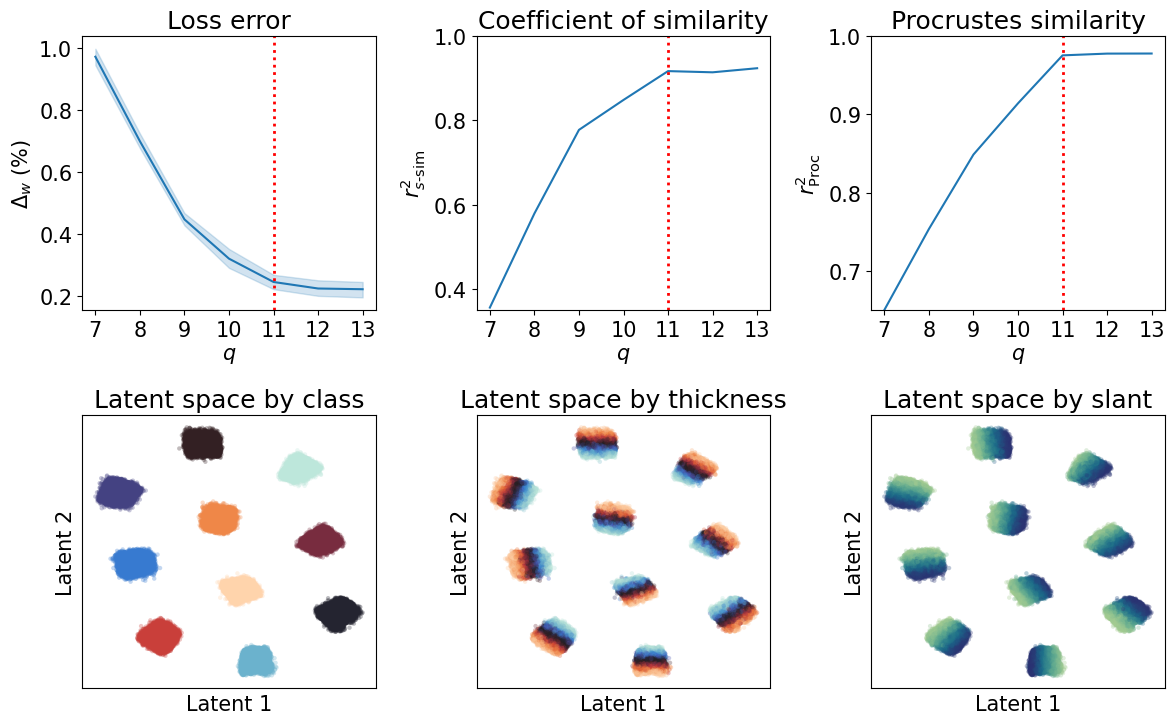}
        \caption{$2D$ Euclidean representations of MNIST using a mixed discrete--continuous weighted InfoNCE loss. Top row: loss error, coefficient of similarity and Procrustes similarity as a function of $q \in [7..13]$. Bottom row: embeddings colored by class, thickness and slant.}
        \label{fig:mnist_11d}
    \end{subfigure}
    \caption{MNIST representations under different weighting schemes.}
    \label{fig:mnist_combined}
\end{figure}

\textbf{Multivariate regression.} We then consider a continuous regression task based on two morphological attributes (thickness and slant). 
Training with weights derived from Euclidean distances in label space yields embeddings that closely approximate label geometry ($r^2=0.91$), with strong Procrustes alignment ($r^2_{\mathrm{Proc}} \approx 0.96$; Fig.~\ref{fig:mnist_2d}).

\textbf{Mixed discrete-continuous labels.} Finally, we study the mixed discrete--continuous regime by combining class labels with morphological features. 
The optimal structure corresponds to the product of a 9-simplex and a unit square, realizable for $q \geq 11$. 
As shown in Fig.~\ref{fig:mnist_11d}, performance improves up to $q=11$. 
In low dimension ($q=2$), the representation is locally isometric within classes ($r^{2,\mathrm{local}}_{\mathrm{Proc}} \approx 0.93$) but not globally ($r^{2,\mathrm{global}}_{\mathrm{Proc}} \approx 0.46$).

Additional $\mathbb{X}$-CLR experiments in Appendix~\ref{ss:cifar}, using both one-hot and semantic label embeddings, further illustrate that consistent label–embedding geometries yield realizable optima.

\subsection{Class imbalance breaks SupCon's regularity, but not Soft SupCon's}\label{sec:class_imbalance}

\begin{wrapfigure}[19]{r}{0.6\textwidth}
    \vspace{-10pt}
    \centering
    \includegraphics[width=\linewidth]{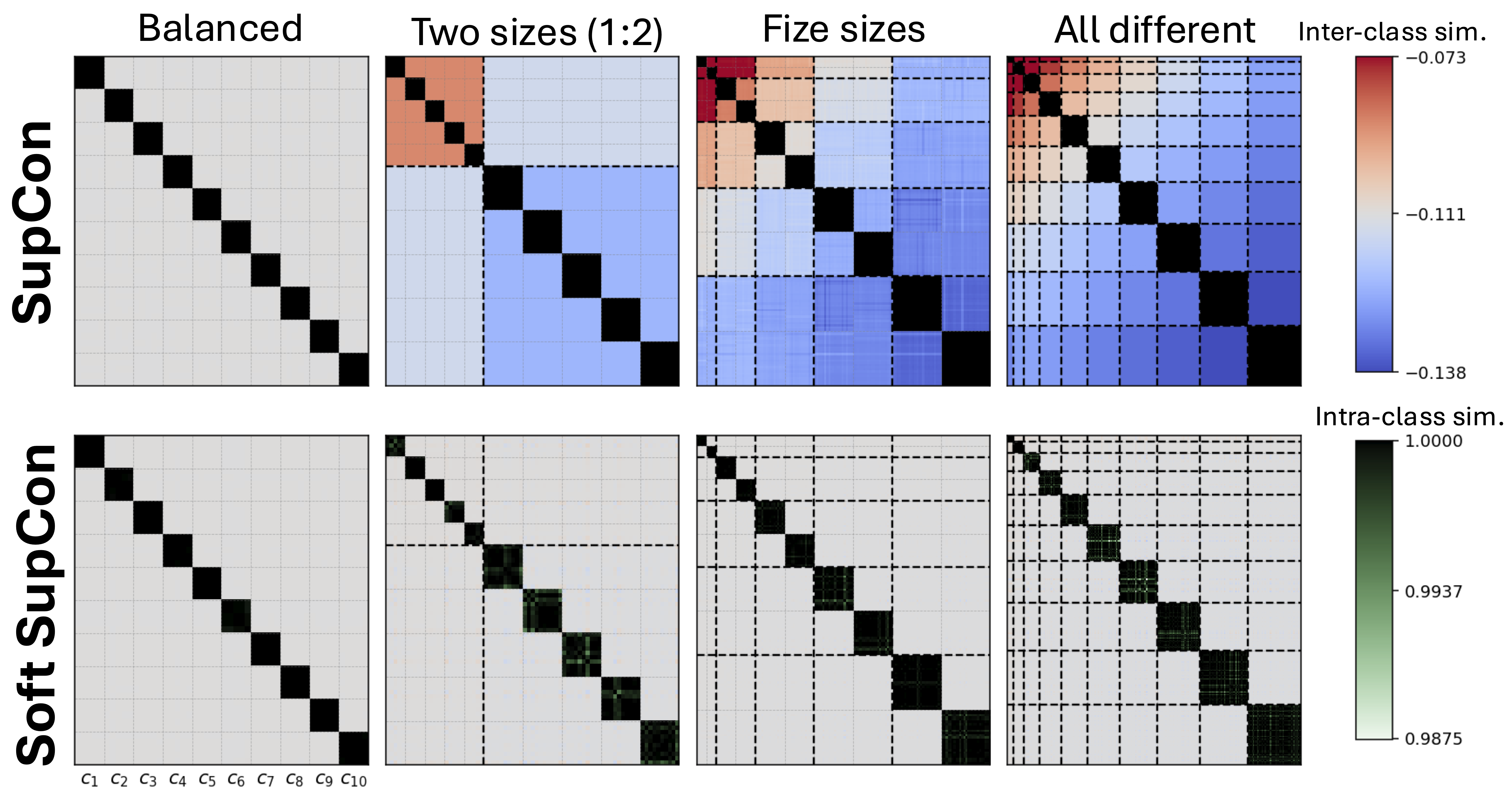}
    \caption{Similarity heatmaps representing $\cos(z_i,z_j)$ for 10-dimensional embeddings learned with SupCon (top) and Soft SupCon (bottom) for  $C=10$ more or less imbalanced classes (columns). Gray off-diagonal blocks indicate an inter-class similarity $\beta^*=-1/(C-1)\approx -0.111$ equivalent to a regular simplex.}
    \vspace{-10pt}
    \label{fig:supcon_imbalanced}
\end{wrapfigure}

To confirm the predictions made in \S~\ref{sec:supcon} regarding Hard and Soft SupCon, we optimized the $w$-InfoNCE loss with a variety of SupCon-like weighting schemes. We synthesized SupCon-like weight matrices (see Appendix~\ref{appdx:sparse}) with \textit{either} hard zeros \textit{or} the soft weight $\varepsilon=e^{-1}$ on the off-diagonal blocks. We considered $C=10$ classes with an assortment of relative sizes ranging from balanced classes to the case where all $10$ classes have a different size, in the proportions $1:2:\cdots:10$. In all cases we observed class collapse. In the soft case we chose the optimal temperature $\tau^*=(10-1)/10$ which, as predicted in \cref{cor:soft_supcon}, yields a regular simplex \textit{even} in the imbalanced case. In contrast, class imbalance breaks the regularity of the optimal SupCon embeddings. The third part of \cref{th:supcon-minimum} is confirmed since for SupCon the inter-class similarity associated with an off-diagonal block can be seen to depend only on that block's shape.

\subsection{Inconsistency of $y$-Aware and its correction}\label{sec:y_aware}

\begin{wrapfigure}[13]{r}{0.5\textwidth}
    \vspace{-10pt}
    \centering
    \includegraphics[width=\linewidth]{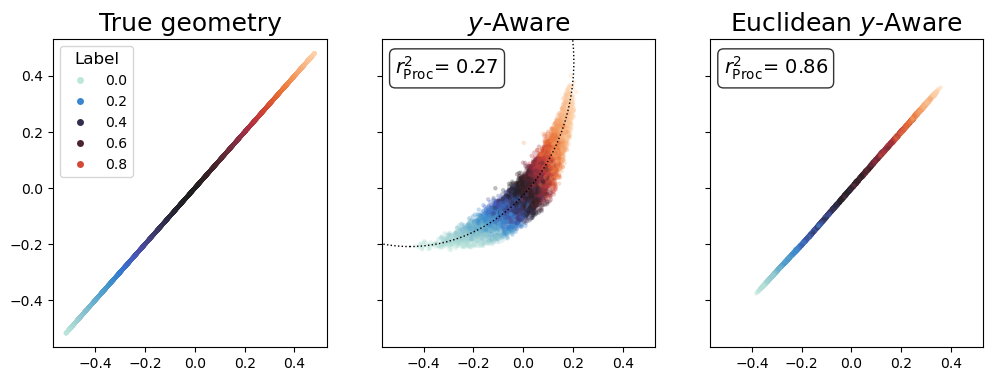}
    \caption{Euclidean representations of MNIST in $\mathbb{R}^2$ using digit's thickness as continuous label. Left: true 1D geometry of the problem. Center: $y$-Aware representation using $\cos$ as similarity function in $\mathcal{Z}$. Right: the corrected $y$-Aware representations learned using Euclidean distance instead.}
    \vspace{-10pt}
    \label{fig:yaware_1d_mnist}
\end{wrapfigure}

We showed in~\cref{th:y_aware} that the $y$-Aware weighting scheme generally does not reach its lower bound unless the labels lie on a hypersphere. To address this, we proposed a solution by using Euclidean distance in both the label and latent spaces. We confirm these predictions on MNIST with line thickness as a continuous label (\S~\ref{sec:w_infonce_mnist}). As shown in~\cref{fig:yaware_1d_mnist}, cosine similarity in $\mathcal{Z}$ constrains the embeddings to a spherical geometry. Replacing it with negative squared Euclidean distance (consistent with the label space) allows the encoder to recover the correct geometry ($r^2_{\mathrm{Proc}}=0.85$).

\section{Discussion and Open Problems}

This work advances the understanding of weighted contrastive learning in both weakly supervised and supervised settings by explicitly characterizing the global optimal representation using a rigorous geometric framework. We derived conditions on weighting schemes to guarantee the existence of optimal realizations and introduced three novel metrics for assessing them. We made several rigorous predictions which we went on to confirm in practical applications.

\textbf{Limitations.} The main limitation of this work is that we assumed an explicit 
dissimilarity matrix encoding prior knowledge about the problem through labels or meta-data. In the self-supervised case (e.g. SimCLR~\cite{chen2020simple}), $D$ is a rough approximation of an \emph{implicit} dissimilarity matrix across samples. Characterizing this matrix is difficult and should use empirical evidence and realistic assumptions~\cite{saunshi2022inductive} as the basis for theoretical results. We leave this as an avenue for future research.

%SimCLR: imposes a regular $n$-simplex in a low-dimensional space. Impossible to achieve when $n$ is very large compared to embedding dimension (typically $n=1$M vs $q=1$k). To by pass this problem, one future work is to see SSL InfoNCE as an approximation of $w$-InfoNCE with particular structure for $w$ based on the encoder and the augmentation strategy. Leave this as future work. 

Overall, we believe that our geometric framework offers a fresh perspective on CL and SSL algorithms and a complementary approach to the mainstream probabilistic and graph-based views. We hope that our theoretical analysis of CL will be used by practitioners to better refine their weighting schemes, as well as by theorists to deepen their understanding of CL and SSL.

% Acknowledgements should only appear in the accepted version.
\bibliographystyle{plainnat}
\bibliography{biblio}

%%%%%%%%%%%%%%%%%%%%%%%%%%%%%%%%%%%%%%%%%%%%%%%%%%%%%%%%%%%%
% APPENDIX

\newpage
\appendix
\section{Entropic lower bound}
\label{proof:optimal_embedding}

\StatementInfonceOptimum*

\begin{proof}
    We define two 2D discrete distributions over $[1..n]^2$ via the density functions 
    \begin{align*}
        p_W(i, j) = \frac{\delta_{i\neq j}w_{ij}}{n\sum_{k \neq i} w_{ik}},\qquad p_S(i,j) = \frac{\delta_{i\neq j}\exp{s_{ij}}}{n\sum_{k \neq i} \exp{s_{ik}}}.
    \end{align*}
   Then it is easy to see that $\mathcal L_\mathrm{NCE}^W$ equals the cross-entropy of $p_S$ relative to $p_W$:

    \begin{equation*}
        \mathcal{L}_{\mathrm{NCE}}^W = H(p_W, p_S) := -\sum_{i, j} p_W(i, j) \log p_S(i, j) %$= H(p_w) + D_{KL}(p_w \| p_s) = H(p_w, p_s)
    \end{equation*}
    where we make use of the convention that $0\log 0 = 0$. By Gibbs' inequality, $H(p_W) \leq H(p_W, p_S)$, where $H(p_W)$ is the entropy of $p_W$, with equality iff $p_W = p_S$. Suppose this is the case, then for each $i \neq j$, $$\frac{w_{ij} }{\sum_{i \neq k} w_{ik}} = \frac{\exp{s_{ij}}}{\sum_{k \neq i} \exp{s_{ik}}},$$ which is equivalent to
    \begin{equation}\label{eq:sw}
        \frac{\exp s_{ij}}{w_{ij}} = \frac{\sum_{k \neq i} \exp{s_{ik}}}{\sum_{i \ne k}w_{ik}}.
    \end{equation}
    Note that the RHS of \cref{eq:sw} does not depend on $j$. But since the LHS is symmetric, the RHS must not depend on $i$ either. This means that (\ref{eq:sw}) is in fact a constant $C$. Since $w_{ij}>0$ when $i\neq j$, $C>0$. Thus $s_{ij} = \log(w_{ij})+c$ where $c = \log C$, as desired. Conversely, one easily verifies that if $s_{ij} = \log(w_{ij})+c$ then $p_W = p_S$, hence $\mathcal L_\mathrm{NCE}^W = H(p_W, p_S)$ is minimal.

    Finally note that the RHS of (\ref{eq:info_nce_min}) is precisely $H(p_W)$.
\end{proof}

\section{Characterizations of the DGP}
\subsection{Euclidean DGP}

We state a classical characterization of EDM matrices and their embedding dimensions. The proof can be found in \cite{alfakih2018euclidean} (Theorems 3.2 and 3.8).
\begin{lemma}\label{th:alfakih}
    A symmetric matrix $D \in \mathbb R^{n \times n}$ with vanishing diagonal entries is an EDM if and only if the centered Gram matrix $B := (-1/2)JDJ$, where $J =I_n-(1/n)E_n$ is the centering matrix, is positive semi-definite (PSD). If such is the case then the embedding dimension $r$ of $D$ is equal to $r=\rank(B)$, and moreover $\rank(D)$ is either $r+1$ or $r+2$.
\end{lemma}

\label{proof:mds_solution}

\StatementDGPSolution*

\begin{proof}
Recall that $E_n$ denotes the $n\times n$ matrix with unit coefficients and $I_n$ the $n\times n$ identity matrix, and $J \defeq I_n - (1/n) E_n$ is the so-called \textit{centering matrix}. Since $\mathcal L_{\mathrm{NCE}}^W$ does not depend on the diagonal of $D$, we may assume WLOG that the diagonal of $D$ is identically zero.

$(\Rightarrow)$ Assume there exists $Z^*\in \mathbb{R}^{n\times q}$ which attains the entropic lower bound. By \cref{corr:infonce_eq_mds} there is a constant $c$ such that $\|z_i-z_j\|^2=d_{ij}+c$ for all $i\neq j$. It follows that $D'=(d_{ij}+c\delta_{i\neq j})=D+c(E_n-I_n)$ is an EDM with embedding dimension $r\le q$. Assume for the sake of a contradiction that $c\neq 0$. Since $D'$ is an EDM with embedding dimension $r$, it follows from Theorem~\ref{th:alfakih} that $r=\rank(JD'J)$. But $J(E_n-I_n)J = -J$, hence $JD'J = JDJ - cJ$. Since $\rank(cJ) = \rank(J) = n-1$ and $\rank(JDJ) \leq \rank(D) \leq q+2$, it follows that
$$r = \rank(JD'J) = \rank(JDJ - cJ) \geq \left| \rank(JDJ) -  \rank(cJ)\right| =(n-1)-(q+2) > q$$
which contradicts the fact that $r \leq q$. This proves that $c=0$, hence $D=D'$ is an EDM with embedding dimension $r \leq q$.

$(\Leftarrow)$ Conversely, if $D$ is an EDM with embedding dimension $r \leq q$ then any realization $Z^* \in \mathbb R^{n \times q}$ solves \cref{eq:mds_pb} with $c=0$.

Finally, note that a solution that any two solutions to Equation~(\ref{eq:mds_pb}) in $\mathbb R^q$ with $c=0$ have the same EDM, hence they differ by a Euclidean isometry.
\end{proof}

\subsection{Spherical DGP}
\label{proof:spherical_mds_solution}

We state without proof a useful characterization of spherical EDMs:

\begin{lemma}
\label{th:spherical_edm}
    An $n \times n$ EDM $D$ of embedding dimension $r\le n-2$ is spherical iff $\rank(D)=r+1$, or equivalently if there exists a scalar $\beta$ such that $\beta E_n-D$ is PSD. Moreover, the smallest $\beta$ with this property is $\beta = 2\rho^2$, where $\rho$ is the radius of $D$.
\end{lemma}
\begin{proof}
     See \cite{alfakih2018euclidean}, Theorem~4.2 and 4.3.
\end{proof}
\StatementSphericalDGPSolution*
\begin{proof}
Assume WLOG that the diagonal of $D$ is identically zero.

   $(\Rightarrow)$ Assume there exists $Z^*$ which attains the entropic lower bound, then by \cref{corr:infonce_eq_spherical_mds} there exists $c' \in \mathbb R$ such that $\cos(z_i,z_j)=-\tau d_{ij}+c'$ for each $i \neq j$. Denote by $\widetilde z_i = z_i/\|z_i\|$ the normalized realizations. Then for each $i \neq j$,
    \begin{align*}
        \|\widetilde z_i - \widetilde z_j\|^2 &= \|\widetilde z_i\|^2+\|\widetilde z_j\|^2-2\langle \widetilde z_i,\widetilde z_j\rangle \\
        &= 1 + 1 - 2\cos(z_i, z_j) \\
        &= 2(\tau d_{ij}+1-c').
    \end{align*}
    Thus, $D':=2\tau D+2(1-c')(E_n-I_n)$ is an EDM with embedding dimension $r \le q$. In particular $\rank(D') \leq r+2\leq q+2$. It follows that $c'=1$, else we would have
    $$q+2\geq \rank(D') \geq |\rank(D)-\rank(E_n-I_n)| \geq (n-1)-(q+1)>q+2,$$ a contradiction. Thus, $c'=1$ and $D=D'/(2\tau)$ is an EDM with embedding dimension $r$.
    
Note that $2E_n-D'=2\widetilde{Z}\widetilde{Z}^T$ since for each $i,j$, $\langle \tilde z_i, \tilde z_j\rangle = \cos(z_i, z_j) = 1 - \tau d_{ij}$. In particular, $2E_n - D'$ is PSD. Since $r \leq q \leq n-2$ (as $n>2q+4$), this means by \cref{th:spherical_edm} that $D$ is spherical, and moreover that its radius $\rho'$ satisfies $\rho' \leq 1$. It follows that $D = D'/2\tau$ is spherical and its radius $\rho$ satisfies $\rho =\rho'/\sqrt{2\tau} \leq 1/\sqrt{2\tau}$. If $r=q$ then the $\widetilde z_i$ are generating points of $D'$, hence $\rho'=1$ and $\rho=1/\sqrt{2\tau}$. 

$(\Leftarrow)$ Conversely, assume that $D$ a spherical EDM with embedding dimension $r\le q$ and radius $\rho$.
    
Assume first that $r=q$ and $\rho=1/\sqrt{2\tau}$. Thus, there exists a spherical realization $z_1,\dots,z_n \in \mathbb{R}^q$ such that $\|z_i-z_j\|^2=d_{ij}$ and $\|z_i\|=1/\sqrt{2\tau}$. It follows that for all $i,j$,
\begin{align*}
    \cos(z_i, z_j) &= 2\tau \langle z_i, z_j\rangle\\
    &= \tau(\|z_i\|^2+\|z_j\|^2-\|z_i-z_j\|^2) \\ 
    &= 1 - \tau d_{ij},
\end{align*}
so $Z$ is a solution to~\cref{eq:spherical_mds_pb} with $c'=1$. 
    
On the other hand, if $r<q$ and $\rho \le 1/\sqrt{2\tau}$ then there exists a realization $p_1,\dots, p_n \in \mathbb{R}^r$ of $D$, with $\|p_i\|=\rho$ and $\|p_i - p_j\|^2 = d_{ij}$. Setting $$z_i := \left(p_i, \sqrt{\frac{1}{2\tau}-\rho^2}, 0,\dots,0\right)\in \mathbb{R}^q$$ which obviously have the same EDM as the $p_i$, and since $\|z_i\|^2=\|p_i\|^2+1/(2\tau)-\rho^2=1/(2\tau)$ we can conclude as before that $Z$ is a solution to~\cref{eq:spherical_mds_pb}.
    
Note finally that any two solutions to~\cref{eq:spherical_mds_pb} must satisfy $c'=1$, hence they have the same cosine matrix, thus their normalizations differ by a linear Euclidean isometry.
\end{proof}

\section{Practical cases}
\subsection{Supervised multivariate regression case}
\label{proof:sup_continuous_case}

\StatementContinuousLabels*

\begin{proof}
    Since $y_i\in \mathbb{R}^\ell$ and $\ell \le q$, $D$ is an EDM with embedding dimension at most $q$, so $\rank(D)\le q+2$. By \cref{th:mds_solution} there exists an essentially unique embedding $Z^*$ which attains the entropic lower bound. The stated $Z^*$ is one such embedding and any other differs from it by a Euclidean isometry.  
\end{proof}

\StatementYAware*

\begin{proof}
   Note that $D=(\|y_i-y_j\|^2)$ is an EDM with embedding dimension $r \leq q$ (because $y_i\in\mathbb{R}^\ell$ and $\ell \le q$) so $\rank(D)\le q+2$. If there exists an embedding that attains the entropic lower bound then according to~\cref{th:spherical_edm}, $D$ is a spherical EDM so the $y_i$ lie on a hypersphere.
\end{proof}

\StatementXCLR*

\begin{proof}
    
   Assume WLOG that the $y_i$ are normalized. Borrowing the notation from \cref{th:spherical_dgp_solution}, let $D=(d_{ij})$ with $d_{ij}=-\cos(y_i, y_j)/\tau'$ for $i,j \in [1..n]$ be the dissimilarity matrix associated to $W$, with $w_{ij}=\exp(-d_{ij})$. Since the $y_i$ are normalized, it holds that
$$d_{ij} = \frac{1}{\tau'}\left(\frac{\|y_i-y_j\|^2}{2}-1\right).$$
Thus, $d'_{ij}:=d_{ij}+1/\tau'$ is a spherical EDM with embedding dimension $r \leq \ell$ and radius $\rho \leq 1/\sqrt{2\tau'}$. Since $D$ and $D'$ differ by an additive constant and owing to \cref{corr:infonce_eq_spherical_mds}, replacing $D$ by $D'$ does not change the set of points that attain the entropic lower bound of the $w$-InfoNCE loss. Since $\rank(D')= \ell+1<q+1$ and $\rho\leq 1/\sqrt{2\tau}$ (as $\tau \leq \tau'$), it follows from \cref{th:spherical_dgp_solution}, that there is an essentially unique solution that attains the entropic lower bound for the dissimilarity matrix $D'$ (and hence for $D$). 

It is easy to see that $z_i^*=(y_i, \sqrt{\tau'/\tau-1}, 0, \ldots, 0)$ is one such solution since it satisfies (\ref{eq:spherical_mds_pb}) with $c'=1 - \tau/\tau'$.

\end{proof}

\section{Sparse weights}\label{appdx:sparse}
Previously, we restricted our attention to the case where $w_{ij} > 0$ for any $i\neq j$ because it leads to a simple characterization of the minimum of $\mathcal L_{\mathrm{NCE}}^w$ (\cref{th:infonce_optimum}). However, relaxing that condition to $w_{ij} \geq 0$ is of interest since in practical applications, $W$ is often sparse. Note, however, that $W$ must contain sufficiently many nonzero entries so as to avoid divisions by $0$ in the expression $\mathcal L_{\mathrm{NCE}}^W$. Specifically, we require that $\sum_{k\neq i} w_{ik} > 0$, which is equivalent to saying that each line or column in $W$ has at least one nonzero off-diagonal entry. Let us call symmetric nonnegatively-valued matrices with this property \textit{well-conditioned}.

We now turn our attention to the sparse weight matrices $W=(w_{ij})$ encountered in practical applications, namely the SupCon loss~\cite{khosla2020supervised}, where $w_{ij}$ is defined as $1$ if samples $i$ and $j$ are in the same class, and $0$ otherwise. By assuming WLOG that a given class is represented by a contiguous set of indices, we can represent $W$ in block form as follows:
\begin{equation}\label{eq:supcon-like}
W = \begin{pmatrix}
\mathbf{1}_{\ell_1 \times \ell_1} & \mathbf{0} & \cdots & \mathbf{0} \\
\mathbf{0} & \mathbf{1}_{\ell_2 \times \ell_2} & \cdots & \mathbf{0} \\
\vdots & \vdots & \ddots & \vdots \\
\mathbf{0} & \mathbf{0} & \cdots & \mathbf{1}_{\ell_C \times \ell_C}
\end{pmatrix} \in \mathbb R^{n \times n},\end{equation}
where $C$ is the number of classes, $\ell_c$ the size of the $c$th class, and $\mathbf 1_{\ell_c \times \ell_c}=E_{\ell_c}$ the $\ell_c \times \ell_c$ matrix of ones. We will always assume that $\ell_c \geq 2$, so that $W$ is well-conditioned. We call matrices of this form \textit{SupCon-like}.

In what follows, we denote by $\mathscr S$ some space of $n \times n$ dissimilarity matrices where $S$ is allowed to vary. We make only one assumption about $\mathscr S$, namely that $\exp(\mathscr S)$, i.e. the set of matrices in $\mathscr S$ transformed pointwise by the exponential function, is bounded. In practice, $\mathscr S$ will be chosen as one of the following sets, which correspond, respectively, to the Euclidean and spherical setting:
\begin{align*}
    \mathscr S_{\mathrm{Eucl.}} := &\left\{-D \mid D \textrm{ an EDM of embedding dimension} \leq q\right\},\\
    \mathscr S_{\mathrm{sph.}} := &\left\{\frac{1}{\tau}G \mid G \textrm{ a cosine matrix of rank} \leq q\right\}.
\end{align*}

\subsection{A condition for the optimality of the entropic lower bound}\label{appdx:sparse_suboptimality}
In the proof of \cref{th:infonce_optimum}, the existence of a minimizer crucially relied on the assumption that $w_{ij}$ has positive entries  in order for $\log(w_{ij})$ to be well-defined. Recall the expression for the minimum value, namely
\begin{equation}\label{eq:info_nce_min_bis}
        H(p_W) = -\frac{1}{n}\sum_{i \neq j} \frac{w_{ij}}{\sum_{k\neq i} w_{ik}} \log\left(\frac{w_{ij}}{\sum_{k \neq i} w_{ik}} \right),
    \end{equation}
where we borrow the notations $p_W$ and $H(-)$ used in the proof of that theorem. Remark that the RHS of (\ref{eq:info_nce_min_bis}) has a well-defined meaning if strict positiveness assumption on $W$ is relaxed and instead we only require that $W$ be well-conditioned. Indeed, $t\log(t) \to 0$ as $t\to 0^+$ so $H(p_W)$ can be evaluated using the convention $0\log(0):= 0$. Thus, we ask ourselves if the lower bound $H(p_W) \leq \mathcal{L}_\textrm{NCE}^W$ continues to hold.

The following theorem answers that question in the positive and shows that it is necessarily sharp when $W$ contains zeros. Moreover, we give a necessary and sufficient condition for optimality of the entropic bound, which we characterize in terms of $\cl(\exp(\mathscr S))$, where $\cl({-})$ denotes the topological closure of a set.
\begin{theorem}\label{lemma:well-conditioned}
Let $W$ be a well-conditioned $n \times n$ matrix.
\begin{enumerate}
    \item The bound $H(p_W) \leq \mathcal L_{\mathrm{NCE}}^W(S)$ holds for all $S \in \mathscr S$.
    \item If $W$ contains at least one off-diagonal zero then that bound is strict.
    \item If $(S^{(m)})_{m \in \mathbb N}$ is a sequence of matrices in $\mathscr S$ such that $\exp(S^{(m)}) \to W$ as $m \to \infty$, then $\mathcal L_{\mathrm{NCE}}^W(S^{(m)}) \to H(p_{W}).$
    \item The bound $H(p_W) \leq \mathcal L_{\mathrm{NCE}}^W(S)$ is optimal iff $cW \in \cl(\exp(\mathscr S))$ up to diagonal elements, for some $c > 0$.
\end{enumerate}
\end{theorem}
\begin{proof}
The first statement follows from Gibbs' inequality, as in the proof of \cref{th:infonce_optimum}. In that proof we saw that the condition for the lower bound being attained amounts to the equality
$$\frac{w_{ij} }{\sum_{i \neq k} w_{ik}} = \frac{\exp{s_{ij}}}{\sum_{k \neq i} \exp{s_{ik}}},$$
for each $i\neq j$. This is plainly impossible if $w_{ij}=0$ (since the RHS is nonzero), in which case the bound is strict.

For the third statement, let $(S^{(m)})_{m\in \mathbb N}$ be a sequence of matrices in $\mathscr S$ such that $W^{(m)} := \exp(S^{(m)}) \to  W$ as $m \to \infty$. Then,
$$\mathcal L_{\textrm{NCE}}^{W} (S^{(m)}) = -\frac{1}{n}\sum_{i\neq j}  \frac{w_{ij}}{\sum_{k\neq i} w_{ik}} \log\left(\frac{w^{(m)}_{ij}}{\sum_{k \neq i} w^{(m)}_{ik}} \right).$$

Denote by $f$ the function defined for $t > 0$ by $f(t)=t\log(t)$ and continuously extended at $t=0$ by setting $f(0)=0$. We want to show that for each $i\neq j$,
$$\frac{w_{ij}}{\sum_{k\neq i} w_{ik}} \log\left(\frac{w^{(m)}_{ij}}{\sum_{k \neq i} w^{(m)}_{ik}}\right) \longrightarrow f\left(\frac{w_{ij}}{\sum_{k\neq i} w_{ik}}\right)\quad \textrm{as } m \to \infty.$$
If $w_{ij}=0$ the statement is trivially true since $f(0)=0$ by definition. On the other hand, if $w_{ij}> 0$ then $\log$ is continuous at $w_{ij}$, so
$$\frac{w_{ij}}{\sum_{k\neq i} w_{ik}} \log\left(\frac{w^{(m)}_{ij}}{\sum_{k \neq i} w^{(m)}_{ik}}\right) \longrightarrow \frac{w_{ij}}{\sum_{k\neq i} w_{ik}} \log\left(\frac{w_{ij}}{\sum_{k \neq i} w_{ik}}\right) = f\left(\frac{w_{ij}}{\sum_{k\neq i} w_{ik}}\right).$$ It follows that
$$\mathcal L_{\textrm{NCE}}^{W} (S^{(m)}) \longrightarrow -\frac{1}{n}\sum_{i\neq j}  f\left(\frac{w_{ij}}{\sum_{k\neq i} w_{ik}} \right)=H(p_W),$$
which proves the third statement.

For the ``if'' part of the fourth statement, let us assume that $cW \in \cl(\exp(\mathscr S))$ up to diagonal elements, for some $c > 0$. Then there is a sequence $(S^{(m)})_{m \in \mathbb N}$ such that $\exp(S^{(m)}) \to cW$, up to diagonal elements. Since $cW$ is well-conditioned, it follows that $\mathcal L_\mathrm{NCE}^W(S^{(m)}) \to H(p_{cW})=H(p_W)$. This proves that the bound is optimal.

Conversely, let us assume that the bound $H(p_W) \leq \mathcal L_\mathrm{NCE}^W(S)$ is optimal, so there exists a sequence $(S^{(m)})_{m \in \mathbb N}$ such that $\mathcal L_\mathrm{NCE}^W(S^{(m)}) \to H(p_W)$. Let $V^{(m)} := \exp(S^{(m)})$. Since $V^{(m)} \in \exp(\mathscr S)$ it is bounded, hence by the Bolzano--Weierstrass theorem there is a convergent subsequence $V^{(m_p)} \to V$. It follows that
\begin{align*}
    H(p_W) &= \lim_{m\to \infty} \mathcal L_\mathrm{NCE}^W(S^{(m)})
    = \lim_{m \to \infty} -\frac{1}{n}\sum_{i \neq j} \frac{w_{ij}}{\sum_{k\neq i} w_{ik}} \log\left(\frac{v^{(m)}_{ij}}{\sum_{k \neq i} v^{(m)}_{ik}} \right)\\
    &=\lim_{p \to \infty} -\frac{1}{n}\sum_{i \neq j} \frac{w_{ij}}{\sum_{k\neq i} w_{ik}} \log\left(\frac{v^{(m_p)}_{ij}}{\sum_{k \neq i} v^{(m_p)}_{ik}} \right)
    =-\frac{1}{n}\sum_{i \neq j} \frac{w_{ij}}{\sum_{k\neq i} w_{ik}} \log\left(\frac{v_{ij}}{\sum_{k \neq i} v_{ik}} \right)\\
    &=H(p_W, p_V),
\end{align*}
where $p_V$ is the 2D discrete sequence defined on $[1..n]^2$ analogously to $p_W$ (see proof of \cref{th:infonce_optimum}). By Gibbs' inequality this implies that $p_W=p_V$, and arguing by symmetry just as we did in \cref{th:infonce_optimum} we see that there exists $c > 0$ such that $V=cW$ up to diagonal elements. Since $V =\lim_{p\to \infty} \exp(V^{(m_p)}) \in \cl(\exp(\mathscr S))$, this completes the proof.
\end{proof}

We now go on to characterize the so-called Euclidean SupCon, i.e. the case where $W$ is a SupCon-like weight matrix (\ref{eq:supcon-like}), and $\mathscr S=\mathscr S_{\mathrm{Eucl.}}$ is the set of negative squared $n \times n$ EDMs of embedding dimension at most $q$. If $W$ encodes $C \geq 2$ classes then it contains zeros, thereby guaranteeing that the lower bound is strict. Moreover, quasi-optima (i.e. configurations of points such that $\mathcal L_{\mathrm{NCE}}^W  \leq H(p_W)+\varepsilon$ for some small $\varepsilon > 0$) can be produced in a wealth of qualitatively distinct geometries. This comes in stark contrast to when we made previously with Soft SupCon, in which case we saw that the optima were usually constrained to a unique configuration, up to Euclidean isometry (\cref{th:mds_solution}).

To see where this non-uniqueness comes from in the Euclidean SupCon case, choose pairwise distinct points $\mu_1, \ldots, \mu_C \in \mathbb R^q$ which will serve as class prototypes. Let $m > 0$ and consider the embedding which assigns each sample $x_i$ to $m\mu_{c_i}$, where $c_i \in [1..C]$ is that sample's class. Denoting by $S^{(m)}$ the negative squared Euclidean distance matrix associated with that embedding, it follows that $\exp(S^{(m)})$ can be expressed in block form as
\begin{equation}
\exp(S^{(m)}) = \begin{pmatrix}
\mathbf{1}_{\ell_1 \times \ell_1} & e^{-m^2 \|\mu_1-\mu_2\|^2} & \cdots & e^{-m^2 \|\mu_1-\mu_C\|^2} \\
e^{-m^2 \|\mu_2-\mu_1\|^2} & \mathbf{1}_{\ell_2 \times \ell_2} & \cdots & e^{-m^2 \|\mu_2-\mu_C\|^2} \\
\vdots & \vdots & \ddots & \vdots \\
e^{-m^2 \|\mu_C-\mu_1\|^2} & e^{-m^2 \|\mu_C-\mu_2\|^2} & \cdots & \mathbf{1}_{\ell_C \times \ell_C}
\end{pmatrix},\end{equation}
where the off-diagonal coefficients represent constant-valued blocks of the appropriate shape. 

Note in particular that $\exp(S^{(m)}) \to W$ as $m \to \infty$, so $\mathcal L_\mathrm{NCE}^W(S^{(m)}) \to H(p_W)$ by \cref{lemma:well-conditioned}. Thus, taking a large enough $m$ yields a near-optimal solution. Yet, the choice of the $\mu_1, \ldots, \mu_C$ was entirely unconstrained, so the class prototypes can be chosen in any number of qualitatively dissimilar configurations. For instance, when $C=3$ the triangle formed by the class prototypes can be chosen to be equilateral, right-angled or scalene and all the while minimize $\mathcal L_{\mathrm{NCE}}^W$ to within, say, $\varepsilon =10^{-10}$ of the theoretical lower bound.

\subsection{Hard SupCon}\label{appdx:supcon}
We now go on to SupCon proper, that is to say the case where $W$ is a SupCon-like binary weight matrix (\ref{eq:supcon-like}) and $\mathscr S = \mathscr S_{\mathrm{sph.}}$ consists of those $n \times n$ matrices of the form $G/\tau$, where $G$ is a cosine matrix of rank at most $q$. Note that $\mathscr S_{\mathrm{sph.}}$ is a compact set so owing to \cref{lemma:well-conditioned}, the bound $H(p_W) \lneq \mathcal L_\mathrm{NCE}^W(S)$ is sharp, and thus \textit{a priori} sub-optimal, since a continuous function on a compact set attains its minimum.

We attack the problem of minimizing $\mathcal L_{\mathrm{NCE}}^W$ by exploiting its convexity as well as the great deal of symmetries of $W$. In what follows we let $\mathscr S_0 \supseteq \mathscr S$ denote the set of $n \times n$ symmetric matrices of the form $G/\tau$ where $G$ is a cosine matrix; in other words it is defined in the same way as $\mathscr S$ except that the rank is condition relaxed. Crucially, $\mathscr S_0$ is a convex set (which in general $\mathscr S$ is not).

We denote by $\mathfrak{S}_n$ the group of permutations of $[1..n]$. Then $\mathfrak S_n$ acts on the space of $n\times n$ matrices by permutation of lines and columns, i.e. if $\sigma  \in \mathfrak S_n$ and $X \in \mathbb R^{n \times n}$ then $\sigma X$ is the $n \times n$ matrix defined by $(\sigma  X)_{ij} = X_{\sigma(i)\sigma(j)}$. Note that $\mathscr S_0$ is closed under this group action, i.e. if $S \in \mathscr S_0$ then $\sigma S \in \mathscr S_0$. Let $\mathfrak S_X$ denote the \textit{stabilizer} of a given matrix $X \in \mathbb R^{n \times n}$, i.e. the subgroup of $\mathfrak S_n$ consisting of those permutations $\sigma$ such that $\sigma X=X$.

We say that $S$ is \textit{$W$-symmetric} if $\mathfrak S_W \subseteq \mathfrak S_S$, i.e. any symmetry of $W$ is also a symmetry of $S$.

\begin{lemma}\label{lemma:symmetrization}
For any $S \in \mathscr S_0$, there exists $S' \in \mathscr S_0$ which is $W$-symmetric such that $\mathcal L_{\mathrm{NCE}}^W(S') \leq \mathcal L_{\mathrm{NCE}}^W(S)$.
\end{lemma}
\begin{proof}
Note that $S \mapsto \mathcal L_{\mathrm{NCE}}^W(S)$ is a convex function on the convex set $\mathscr S_0$; indeed this can be seen by expressing it as a nonnegatively-weighted sum of convex functions:
$$\mathcal L_{\mathrm{NCE}}^W(S)= \sum_{i\neq j} \frac{w_{ij}}{n\sum_{k\neq i} w_{ik}} \left(  \mathrm{LSE}_{k\neq i} (s_{ik}) - s_{ij}\right),$$
where $\mathrm{LSE}$ denotes the convex log-sum-exp operation, i.e. $\mathrm{LSE}_k(x_k) = \log(\sum_k \exp(x_k))$. Now, we define the symmetrized version of $S$ as
$$S' := \frac{1}{|\mathfrak{S}_W|}\sum_{\sigma \in \mathfrak{S}_W} \sigma S \in \mathscr S_0.$$ It follows from the convexity of $\mathcal L_\mathrm{NCE}^W$ that
\begin{align}\label{eq:group_avg}
\mathcal L_\mathrm{NCE}^W(S') \leq \frac{1}{|\mathfrak{S}_W|}\sum_{\sigma \in \mathfrak{S}_W} \mathcal L_\mathrm{NCE}^W(\sigma S).
\end{align}
Note that the RHS is none else than $\mathcal L_\mathrm{NCE}^W(S)$. Indeed, it is an exercise in sum notation to see that $\mathcal L_\mathrm{NCE}^W(\sigma S) = \mathcal L_\mathrm{NCE}^{\sigma^{-1}W}(S)$ holds for any $\sigma \in \mathfrak S_n$. If $\sigma \in \mathfrak S_W$ then $\sigma^{-1} \in \mathfrak S_W$, i.e. $\sigma^{-1}W=W$, so $\mathcal L_\mathrm{NCE}^W(\sigma S) = \mathcal L_\mathrm{NCE}^{W}(S)$. Plugging that into (\ref{eq:group_avg}) yields $\mathcal L_\mathrm{NCE}^W(S') \leq \mathcal L_\mathrm{NCE}^W(S)$. 

It remains to show that $S'$ is $W$-symmetric. To this end, note that for any $\pi \in \mathfrak{S}_W$,
$$\pi S' = \frac{1}{|\mathfrak S_W|}\sum_{\sigma \in \mathfrak{S}_W} \pi (\sigma S) = \frac{1}{|\mathfrak S_W|}\sum_{\sigma \in \mathfrak{S}_W} (\pi \sigma) S = \frac{1}{|\mathfrak S_W|}\sum_{\sigma' \in \mathfrak{S}_W} \sigma' S = S'$$
where we re-indexed the sum in $\sigma' = \pi \sigma$, using the fact that $\mathfrak{S}_W$ is a subgroup of $\mathfrak{S}_n$.  Thus, $\pi \in \mathfrak{S}_{S'}$, as required.
\end{proof}
The following lemma elucidates the structure of $W$-symmetric matrices.
\begin{lemma}\label{lemma:blockwise}
A $W$-symmetric matrix $S$ is constant on the blocks induced by the class partition, up to diagonal elements. Any two off-diagonal blocks with the same shape are equal. If two diagonal blocks have the same shape then their off-diagonal values are equal.
\end{lemma}
\begin{proof}
Let $i \neq j$ and $i' \neq j'$ be two pairs of distinct indices in the same block, that is to say such that the pairs $i,i'$ and $j,j'$ are both in the same class. Suppose first that $i$ and $j$ belong to the same class $I \subseteq [1..n]$. Since $\mathfrak S_I$ acts 2-transitively on $I$, there exists $\sigma \in \mathfrak S_I$ such that $\sigma(i) = i'$ and $\sigma(j) = j'$. We extend $\sigma$ to $\mathfrak{S}_n$ such that $\sigma$ acts like the identity on $[1..n]\setminus I$. It follows that $\sigma \in \mathfrak S_W$. It follows that $s_{ij} = s_{\sigma(i)\sigma(j)} = s_{i'j'}$ since $S$ is $W$-symmetric. On the other hand, if $i$ and $j$ belong to distinct classes $I$ and $J$, the transpositions $(i\ j)$ and $(i'\ j')$ have disjoint support and composing them yields a permutation $\sigma \in \mathfrak{S}_W$, and again we get $s_{ij}=s_{\sigma(i)\sigma(j)} = s_{i'j'}$ since $S$ is $W$-symmetric. This shows that off-diagonal elements in the same block are equal.

Let $(I,J)$ and $(I',J')$ be pairs of classes representing off-diagonal blocks of the same shape, i.e. $|I|=|I'|$ and $|J|=|J'|$. Then there are permutations $\sigma_1 : I \to I'$ and $\sigma_2 : J \to J'$. Let $I_3, \ldots, I_{C}$ denote the classes not equal to $I$ or $J$, and let $I_3', \ldots, I_{C}'$ those not equal to $I'$ or $J'$. The latter may be chosen in such an order that $|I_i| = |I'_i|$ for $3 \leq i \leq C$ so there exist bijections $\sigma_i : I_i \to I_i'$. The bijections $\sigma_1, \sigma_2, \sigma_3, \ldots, \sigma_C$, whose domains and codomains are all disjoint, combine into a single permutation $\sigma : I \sqcup J \sqcup I_3 \sqcup \cdots \sqcup I_c =[1..n] \to I' \sqcup J' \sqcup I_3' \sqcup \cdots \sqcup I_c' = [1..n]$ which permutes classes, and therefore belongs to $\mathfrak S_W$. It follows that $\sigma S = S$, hence $S_{IJ} = S_{\sigma(I) \sigma(J)} = S_{I'J'}$. This proves that the value of $S$ at an off-diagonal blocks depends only on its shape, as required.

Finally, if $(I,I)$ and $(J,J)$ are diagonal blocks such that $|I|=|J|$, a bijection $\sigma : I \to J$ can be extended using the same idea as above to produce $\sigma \in \mathfrak S_W$ such that $\sigma(I)=J$. Hence for $i, i' \in I$, if $i\neq i'$ then $s_{II} = s_{ii'} = s_{\sigma(i)\sigma(i')}= s_{JJ}$.
\end{proof}
\begin{lemma}\label{lemma:lse-strict}
    Let $\mathbf a, \mathbf v\in \mathbb R^m$. Then the restriction of the log-sum-exp function $\LSE: \mathbb R^m \to \mathbb R$ to the line $\mathbf a+\mathbb R\mathbf v \subseteq \mathbb R^m$ is convex, and it is strictly convex iff the coefficients of $\mathbf v$ are not identical.
\end{lemma}
We omit the proof which can easily be derived using elementary calculus and the Cauchy--Schwarz inequality.

We say that a matrix in $\mathscr S_0$ is \textit{$W$-collapsed} if it is $W$-symmetric and the diagonal entries agree with their respective diagonal blocks, which is to say that the values of the diagonal blocks are identically and maximally equal to $1/\tau$. A $W$-collapsed matrix can be expressed in block form as $(\beta_{cc'} 
\cdot \mathbf 1_{\ell_c \times \ell_{c'}})_{c,c'=1}^C$ where $\beta_{cc}=1/\tau$ for each $c\in [1..C]$. \cref{lemma:blockwise} tells us that if $c\neq c'$ then $\beta_{cc'}$ depends only on the shape of the block, i.e. $(\ell_c, \ell_{c'})$.

\begin{lemma}\label{th:minima_symmetric}
The function $\mathcal L_{\mathrm{NCE}}^W$ has a unique global minimum over $\mathscr S_0$. Moreover, that minimum is $W$-collapsed.
\end{lemma}
\begin{proof}
Let $S \in \mathscr S_0$ be a global minimum, which exists by compactness. By \cref{lemma:symmetrization}, there is a $W$-symmetric $S' \in \mathscr S_0$ such that $\mathcal L_{\mathrm{NCE}}^W(S') \leq \mathcal L_{\mathrm{NCE}}^W(S)$. But since $S$ is a global minimum over $\mathscr S_0$, so must be $S'$. Since the minimizers of a convex function over a convex set form a convex set, $\mathcal L_{\mathrm{NCE}}^W$ is identically equal to its minimum on the line segment $[S,S'] \subseteq \mathscr S_0$, so for any $t \in [0,1]$ it holds that $\mathcal L_{\mathrm{NCE}}^W(tS+(1-t)S') = t\mathcal L_{\mathrm{NCE}}^W(S) + (1-t)\mathcal L_{\mathrm{NCE}}^W(S')$. Expressing $\mathcal L_{\mathrm{NCE}}^W$ in terms of the $\LSE$ as at the start of the proof of \cref{lemma:symmetrization}, this implies that:
\begin{equation}\label{eq:lse-convex}
    \sum_{i\neq j} \frac{w_{ij}}{\sum_{k \neq i}w_{ik}} \LSE_{k \neq i}(ts_{ik}+ (1-t)s'_{ik}) = \sum_{i\neq j} \frac{w_{ij}}{\sum_{k \neq i}w_{ik}}\left(t \LSE_{k \neq i}(s_{ik}) + (1-t)\LSE_{k\neq i} (s'_{ik})\right).
\end{equation}
For each $i \in [1..n]$, the family $(s_{ik} - s'_{ik})_{k \neq i}$ must be a constant $c_i$, else by \cref{lemma:lse-strict} the strict convexity of $\LSE$ restricted to a line would cause (\ref{eq:lse-convex}) to be a strict inequality. Since $S$ and $S'$ are symmetric, $c_i=c$ does not depend on $i$. It follows that $S=S'+c(E_n - I_n)$ which is $W$-symmetric since $S', E_n$ and $I_n$  all are. By \cref{lemma:blockwise}, $S$ can be expressed simply in terms of the blocks induced by the class partition. It remains to be seen that $S$ is collapsed, i.e. the values of the diagonal blocks are identical and maximally equal to $1/\tau$. Let $(\alpha_c)_{c=1}^C$ denote the values of the diagonal blocks, i.e. \textit{intra}-class similarities, and $(\beta_{cc'})_{c\neq c'}^C$ the \textit{inter}-class similarities. Let $\mathcal I_c$ denote the set of indices belonging to class $c$, and $\ell_c = |\mathcal I_c|$ its size. Then,
\begin{align*}
    \mathcal L_{\mathrm {NCE}}^W(S) = -\frac{1}{n} \sum_{c=1}^C \frac{1}{\ell_c -1} \sum_{\substack{i,j \in \mathcal I_c\\ i \neq j}} \left(\alpha_c- \log\left((\ell_c-1)e^{\alpha_c} + \sum_{c'\neq c}\ell_{c'}e^{\beta_{cc'}}\right)\right).
\end{align*}
For each $c$, this is a strictly decreasing function of $\alpha_c$. Indeed,
$$\frac{\partial \mathcal L_{\mathrm {NCE}}}{\partial \alpha_c} = -\frac{1}{n} \frac{1}{\ell_c -1} \sum_{\substack{i,j \in \mathcal I_c\\ i \neq j}} \frac{\sum_{c'\neq c}\ell_{c'}e^{\beta_{cc'}}}{(\ell_c-1)e^{\alpha_c} + \sum_{c'\neq c}\ell_{c'}e^{\beta_{cc'}}}<0.$$
Because of its block structure, replacing the diagonal entries of $S$ with $1/\tau$ yields a valid matrix in $\mathscr S_0$.
By minimality of $\mathcal L_{\mathrm{NCE}}^W(S)$, it follows that $S$ must have the maximum value of $\alpha_c=1/\tau$ on the diagonal blocks, which shows that $S$ is $W$-collapsed.

We have shown that any global minimum is $W$-collapsed. As to uniqueness, note that if $\widetilde S$ is another global minimum then the same reasoning we carried out on $S'$ shows that $S=\widetilde S+c(E_n-I_n)$ for some constant $c$, which implies that $c=0$ as both $S$ and $\widetilde S$ are $W$-collapsed.
\end{proof}

\StatementSupconMinimum*

\begin{proof}
By \cref{th:minima_symmetric}, $\mathcal L_{\mathrm {NCE}}^W$ has a unique global minimum of $S$ over $\mathscr S_0$. Moreover, $S$ is $W$-collapsed, so in particular it is expressible in block form in terms of $C \times C$ constant blocks with shapes $\ell_c \times \ell_{c'}$ for $c,c' \in [1..C]$. Because of its constant block structure, $\rank(S) \leq C \leq q$, so $S \in \mathscr S$ and it is \textit{a fortiori} the unique global minimum of $\mathcal L_{\mathrm {NCE}}^W$ over $\mathscr S$. The existence of a geometric realization in terms of $C$ class prototypes follows from its constant block structure. Their essential uniqueness follows from the fact that points on a sphere are determined up to linear isometry by their cosine matrix. Finally, 3. follows from \cref{lemma:blockwise}.
\end{proof}
The upshot of \cref{th:supcon-minimum} is that regularity in class sizes can be exploited to greatly reduce the number of parameters in the optimization of the SupCon loss. For instance, we easily recover a following result from~\cite{graf2021dissecting} in the case of balanced classes. Recall that a vertices $\mu_1, \ldots, \mu_C \in \mathbb S^{q-1}$ are said to form a \textit{regular simplex} if $\mu_1+\cdots+\mu_C=\mathbf 0$ and the value of $\cos(\mu_c, \mu_{c'})$ is the same for all $c\neq c'$. This is equivalent to the condition $\cos(\mu_c, \mu_{c'})=1/(C-1)$ for all $c \neq c'$.
\begin{theorem}[Graf et al, 2021]\label{th:graf2021}
    If $C>1$ and $\ell_1=\cdots=\ell_C$ then class prototypes $\mu_1, \ldots, \mu_C \in \mathbb S^{q-1}$ which minimize the SupCon loss form the vertices of a regular simplex.
\end{theorem}
\begin{proof}
    Let $\mu_1, \ldots, \mu_C$ be the optimal class prototypes. Per \cref{th:supcon-minimum} and since there is only one possible class size, $\cos(\mu_c, \mu_{c'})=\beta$ for some constant $\beta$, for every $c \neq c'$. It remains to show that $\beta=-1/(C-1)$. Expressing the minimal value of the loss in terms of $\beta$ at the optimum $S^*=(\cos(\mu_c, \mu_{c'})/\tau \cdot \mathbf 1_{\ell_c \times \ell_{c'}})_{c,c'=1}^C$, we see that
    \begin{align}\label{eq:supcon-beta}
        \mathcal L_{\mathrm{NCE}}^W(S^*)=-\frac1n \sum_{c=1}^C\frac{1}{\ell-1} \sum_{\substack{i,j \in \mathcal I_c\\ i \neq j}}\left(\frac{1}{\tau}-\log\left((\ell-1)e^{1/\tau}+(C-1)\ell e^{\beta/\tau}\right)\right),
        \end{align}
    where $\ell$ denotes the size of every class and $\mathcal I_c \subset [1..n]$ the set of indices belonging to class $c$. In particular we see that the RHS of (\ref{eq:supcon-beta}) is a strictly increasing function of $\beta$ so the value of $\beta$ must be as small as possible subject to the constraint that $S^*$ be PSD. Owing to the block structure of $S^*$, this amounts to the condition that $A_\beta:=I_C +\beta (E_C-I_C)$ is PSD. We can express the quadratic form associated with $A_\beta$ as follows:
    $$\mathbf x^\top A_\beta \mathbf x=(1-\beta)\sum_{c=1}^C x_c^2 + \beta \left(\sum_{c=1}^C  x_c\right)^2.$$ On the one hand, if $\beta=-1/(C-1)$ then $\mathbf x^\top A_\beta \mathbf x \geq 0$ by the Cauchy--Schwarz inequality, so $A_\beta$ is PSD. On the other hand, if $\beta <-1/(C-1)$ and $\mathbf x=(1,1,\ldots,1)$ then $\mathbf x^\top A_\beta \mathbf x=(1-\beta)C+\beta C^2=C(1+(C-1)\beta)<0$ so $A$ is not PSD.
    
    Thus, the optimal inter-class similarity is $\beta=-1/(C-1)$, as desired.
    %  . It follows that
    % $$\left\|\sum_{c=1}^C \mu_c \right\|^2=\sum_{c=1}^C \|\mu_c\|^2 + \sum_{c \neq c'} \langle \mu_c, \mu_{c'}\rangle=C+C(C-1)\beta=0,$$
    % hence $\mu_1+\cdots+\mu_C=\mathbf 0$ so the $\mu_c$ form the vertices of a regular simplex, as desired.
\end{proof}

\section{$\mathbb{X}$-CLR geometry}
\label{ss:cifar}

\begin{figure}[h]
\centering
\begin{minipage}[c]{0.48\textwidth}
    \centering
    \includegraphics[width=\linewidth]{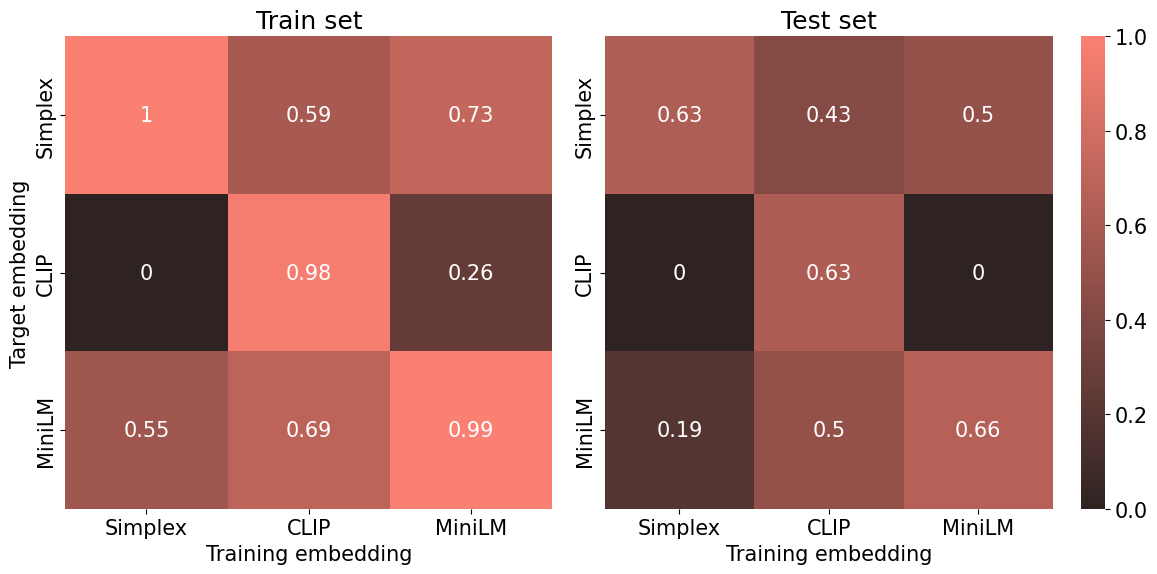}
    \caption{Alignment between learned and target label geometries. Heatmap of $\max(0, r^2_{\mathrm{Proc}}(Z, Z'))$ where columns correspond to the geometry used for training and rows to evaluation geometries. Diagonal entries indicate successful recovery of the target geometry, while off-diagonal values highlight geometric mismatch.}
    \label{fig:cifar}
\end{minipage}
\hfill
\begin{minipage}[c]{0.48\textwidth}
    \centering
    \resizebox{\linewidth}{!}{
    \begin{tabular}{lcccc}
    \toprule
    Target & $r^{2, \mathrm{test}}_{\text{Proc}}$ & $\Delta_W^{\mathrm{test}}$ (\%) & Top-1 (\%) & Top-3 (\%)\\
    \hline
    Simplex & 0.63 & 0.083 & 79.68 & 89.66 \\
    CLIP    & 0.63 & 0.090 & 73.38 & 86.67 \\
    MiniLM  & 0.66 & 0.038 & 78.37 & 89.15 \\
    \bottomrule
    \end{tabular}
    }
    \captionof{table}{Generalization statistics for $\mathbb{X}$-CLR models trained on CIFAR-100. Top-$k$ accuracy is computed on the test set using a cross-validated ridge regression fitted on training representations.}
    \label{tab:cifar}
\end{minipage}
\end{figure}

We evaluate the $\mathbb{X}$-CLR contrastive loss \cite{sobal2024} on CIFAR-100~\cite{cifar100} using a ResNet50 encoder with a $q=512$-dimensional latent space. Labels $y \in \mathbb{R}^\ell$ are encoded in three ways: (i) one-hot class vectors (simplex, $\ell=100$), (ii) CLIP ViT-B/32 text embeddings of the prompt “a photo of \texttt{<class>}” ($\ell=512$), and (iii) MiniLM\footnote{\url{https://huggingface.co/sentence-transformers/all-MiniLM-L6-v2}} embeddings of the same prompt ($\ell=384$).

For each model, we compute $r^2_{\mathrm{Proc}}(Z, Z^*)$, where $Z$ are the learned embeddings and $Z^*$ the target label embeddings used during training (zero-padded if needed). As a control, we also evaluate $r^2_{\mathrm{Proc}}(Z, Z')$ with alternative label geometries $Z'$. Results in~\cref{fig:cifar} show that the models recover their target geometry and that the different label geometries are mutually distinct.

Finally, we tested whether our proposed metrics could explain better the generalization properties of the pre-trained models. To do so, we computed the Procrustes and loss error for the three previous models pre-trained on three different labels geometries and we report them in~\cref{tab:cifar} along with their top-1 and top-3 accuracies for predicting the class labels of CIFAR-100 (seen as generalization error). Strikingly, the $\mathbb{X}$-CLR pre-trained with MiniLM geometry generalizes very well while having the best Procrustes and loss error. It suggests that strong performance on our metrics translates into better model generalization on new tasks.

\section{Connection with Kernel PCA and SNE}
\subsection{Kernel PCA}

We state a special case that gives a precise connection between Kernel PCA and $w$-InfoNCE:

\begin{restatable}[Kernel $w$-InfoNCE]{theorem}{StatementKernelInfoNCE}
    \label{th:infonce_eq_kpca}
    Let $K=(k_{ij})$ be an $n\times n$ positive semi-definite  kernel matrix inducing the squared kernel distance $d_{ij}:=k_{ii} + k_{jj} - 2k_{ij}$. Assume the rank of the centered kernel $\widetilde K$ is at most $q$. The optimal embedding minimizing $w$-InfoNCE with $w_{ij}=\exp(-d_{ij})$ and $s_{ij}=-\|z_i-z_j\|^2$ corresponds to the kernel principal components of $K$, up to Euclidean isometry.
\end{restatable}

\begin{proof}
    Since $K$ is a symmetric positive-definite kernel, $D=(d_{ij})$ is an EDM with embedding dimension $r=\rank(\widetilde{K})\le q$. It follows that $\rank(D)\le q+2$ and one realization of this EDM is exactly the principal components of $K$, i.e. the eigenvectors of $\widetilde{K}$ scaled by the square root of its eigenvalues as per the to the characterization of EDMs in~\cite{alfakih2018euclidean}. We then conclude with~\cref{th:alfakih}. 
\end{proof}

\subsection{SNE}

Stochastic Neighbor Embedding (SNE)~\cite{hinton2002stochastic} is the predecessor of t-SNE and was originally proposed for data visualization. It learns an embedding \(Z\) of high-dimensional data \(X\) by minimizing a $w$-InfoNCE objective with weights $w_{ij}=\exp\!\left(-\|x_i-x_j\|^2/(2\sigma_i^2)\right)$, and similarities $s_{ij}=-\|z_i-z_j\|^2$. This objective violates Assumption~\ref{hyp:symmetry}, since in general \(w_{ij}\neq w_{ji}\). The asymmetry arises from non-uniform local densities in $X$, which induce point-dependent variance $\sigma_i$. If instead one assumes a constant variance $\sigma_i=\sigma$, the SNE objective enforces
$\|z_i-z_j\|\approx\|x_i-x_j\|/\sigma$, as shown in Corollary~\ref{corr:infonce_eq_mds}. 

In this setting, SNE is equivalent to $w$-InfoNCE and recovers an approximation of the principal components of $X$. This approximation becomes exact when the rank of the covariance matrix of $X$ is smaller than the number of components (i.e. embedding dimension), as shown in~\cref{th:infonce_eq_kpca}. 

Below, we show that in that setting, PCA and SNE/$w$-InfoNCE representations are indeed quasi-isometric.

\begin{figure}[h]
    \centering
    \includegraphics[width=.6\linewidth]{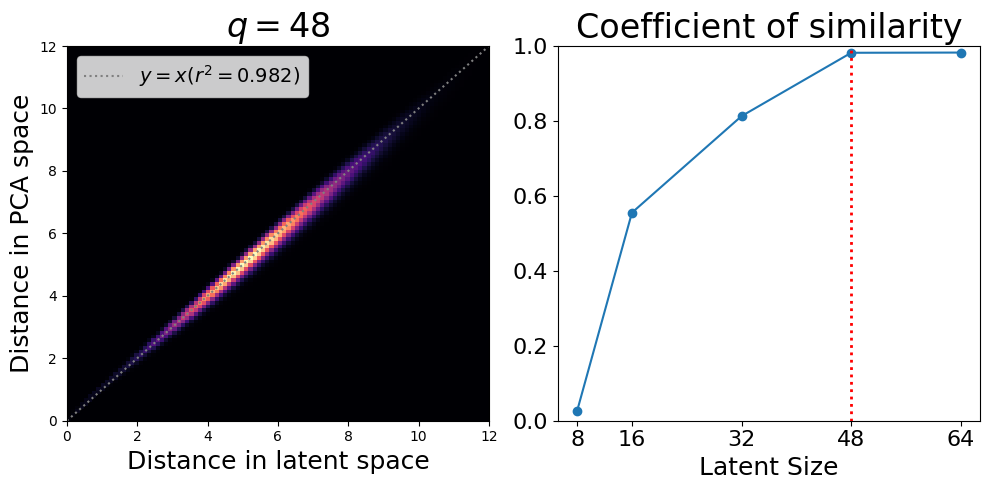}
    \caption{PCA against weighted InfoNCE representations of MNIST with $q$ latent dimensions (or \textit{components}) using Euclidean input distance as dissimilarity measure.}
    \label{fig:pca_vs_winfonce}
\end{figure}

We train a ConvNet on MNIST with latent dimension $q\in \{8, 16, 32, 48,64\}$ and $\sigma=10^3$. In~\cref{fig:pca_vs_winfonce} we report the coefficient of similarity $r^2_{s\textrm{-sim}}$ between $w$-InfoNCE representations and PCA of the test set with the same latent size $q$. We observe that $q=48$ gives an almost perfect score $r^2_{s\textrm{-sim}}=0.98$, which is also confirmed visually when plotting the distances in the ConvNet latent space against PCA space. When $q=48$, PCA explains 90\% of variance, suggesting that $r \approx q$ in this case, confirming our theory.

\section{Implementation details}\label{appdx:exp_details}

The experiments in this paper were carried out in Python using the PyTorch package and a single NVIDIA RTX~4500 Graphics Processing Unit (GPU) with 24~GB of memory.

The code required to reproduce all the experiments in this paper can be found in the anonymized repository with the following URL:

\begin{center}
\url{https://anonymous.4open.science/r/neurips26_geometric_cl-5C16/}
\end{center}

Key details and parameter values are summarized in the following table.
\begin{table}[ht]
\centering
\scriptsize
\setlength{\tabcolsep}{4pt}
\renewcommand{\arraystretch}{1.2}
\begin{tabular}{c|ccccc}
\textbf{\S} & \ref{sec:w_infonce_mnist} & \ref{sec:class_imbalance} & \ref{sec:y_aware} & Appendix~\ref{ss:cifar} \\
\toprule
Dataset        & MNIST & N/A & MNIST & CIFAR-100 \\\hline
File(s) & \makecell[l]{\texttt{mnist\_1d.ipynb}\\ \texttt{mnist\_nd.ipynb}} & \texttt{imbalanced\_supcon.py} & \texttt{mnist\_1d.ipynb} & \texttt{experiments/cifar100/} \\\hline
Architecture        & 3-layer ConvNet & N/A & 3-layer ConvNet & ResNet-50 \\\hline
Loss           & Eucl.\ $w$-InfoNCE & Hard/Soft SupCon & $y$-Aware & $\mathbb X$-CLR \\\hline
Optimizer      & Adam & Adam & Adam & Adam \\\hline
LR             & \texttt{1e-3} & \texttt{5e-2} & \texttt{1e-3} & \texttt{1e-4} \\\hline
WD             & \texttt{2e-6} & N/A & \texttt{2e-6} & \texttt{1e-6} \\\hline
Epochs         & $20$--$50$ & $<20\mathrm{k}$ & $50$ & 150 \\\hline
Batch size     & 512 & full & 512 & 512 \\\hline
Key parameters
  & \makecell[l]{$d=2$ or $d \in [5..12]$,\\
 }
  & \makecell[l]{Hard: $\tau\!=\!0.1$
  \\Soft: $\varepsilon\!=\!e^{-1}$, $\tau=\tau^*$\\$d=C\!=\!10$}
  & \makecell[l]{$d\!=\!2$, $\tau=0.1$}
  & \makecell[l]{$q\!=\!512$,\\$\ell\!\in\!\{100,512,384\}$\\ $\tau=0.1$\\} \\
\vspace{1cm}
\end{tabular}
\caption{%
  Implementation details.
}
\end{table}

%%%%%%%%%%%%%%%%%%%%%%%%%%%%%%%%%%%%%%%%%%%%%%%%%%%%%%%%%%%%
% CHECKLIST

%\newpage
%\input{checklist.tex}

\end{document}